\documentclass{article}

\usepackage{microtype}
\usepackage{graphicx}
\usepackage{subfig}
\usepackage{booktabs} 

\usepackage{hyperref}

\usepackage[accepted]{icml2023}

\usepackage{amsmath}
\usepackage{amssymb}
\usepackage{mathtools}
\usepackage{amsthm}
\usepackage{dsfont}

\usepackage[capitalize,noabbrev]{cleveref}

\theoremstyle{plain}
\newtheorem{thm}{Theorem}[section]

\newtheorem{defn}{Definition}[section]
\newtheorem{lem}{Lemma}[section]
\newtheorem{rem}{Remark}[section]

\newtheorem{assump}{Assumption}[section]

\newtheorem{problem}{Problem}[section]

\usepackage[textsize=tiny]{todonotes}

\newcommand{\R}{\mathbb{R}}
\newcommand{\RR}{\mathbb{R}}

\newcommand{\bX}{\mathds{X}}

\newcommand{\Lip}{\mathrm{Lip}}

\newcommand{\Law}{\mathrm{Law}}
\newcommand{\eps}{\epsilon}
\newcommand{\ud}{\,\mathrm{d}}

\newcommand{\EE}{\mathbb{E}}

\newcommand{\PP}{\mathbb{P}}

\newcommand{\mc}[1]{\mathcal{#1}}

\icmltitlerunning{Directed Chain Generative Adversarial Networks}

\begin{document}

\twocolumn[
\icmltitle{Directed Chain Generative Adversarial Networks}



\icmlsetsymbol{equal}{*}

\begin{icmlauthorlist}
\icmlauthor{Ming Min}{equal,pstat}
\icmlauthor{Ruimeng Hu}{equal,pstat,math}
\icmlauthor{Tomoyuki Ichiba}{pstat}
\end{icmlauthorlist}

\icmlaffiliation{pstat}{Department of Statistics and Applied Probability, University of California, Santa Barbara, CA 93106-3110, USA.}
\icmlaffiliation{math}{Department of Mathematics, University of California, Santa Barbara, CA 93106-3080, USA.}

\icmlcorrespondingauthor{Ming Min}{m\_min@pstat.ucsb.edu}

\icmlkeywords{Generative adversarial networks, Directed chain stochastic differential equations, Multimodal time series}

\vskip 0.3in
]



\printAffiliationsAndNotice{\icmlEqualContribution} 

\begin{abstract}
Real-world data can be multimodal distributed, e.g., data describing the opinion divergence in a community, the interspike interval distribution of neurons, and the oscillators’ natural frequencies. Generating multimodal distributed real-world data has become a challenge to existing generative adversarial networks (GANs). 
For example, it is often observed that Neural SDEs have only demonstrated successful performance mainly in generating unimodal time series datasets.
In this paper, we propose a novel time series generator, named directed chain GANs (DC-GANs), which inserts a time series dataset (called a neighborhood process of the directed chain or input) into the drift and diffusion coefficients of the directed chain SDEs with distributional constraints. DC-GANs can generate new time series of the same distribution as the neighborhood process, and the neighborhood process will provide the key step in learning and generating multimodal distributed time series. The proposed DC-GANs are examined on four datasets, including two stochastic models from social sciences and computational neuroscience, and two real-world datasets on stock prices and energy consumption. To our best knowledge, DC-GANs are the first work that can generate multimodal time series data and consistently outperforms state-of-the-art benchmarks with respect to measures of distribution, data similarity, and predictive ability.
\end{abstract}


\section{INTRODUCTION}

Generative models are important to overcome the limitation of data scarcity, privacy, and costs. In particular, medical data are not easy to get, use or share, due to privacy; and financial time series data are inadequate due to their nonstationarity nature. Times-series generative models, instead of seeking to learn the governing equations
from real data, aim to discover and learn data automatically, and output new data that plausibly can be drawn from the original dataset. Some existing infinite-dimensional generative adversarial networks (GANs) (e.g., \citet{kidger2021neural,li2022tts}) showed successful performance in unimodal time series datasets.
However, many real-world phenomena are multimodal distributed, e.g., data describing the opinion divergence in a community \citep{tsang2014opinion}, the interspike interval distribution \citep{sharma2018suppression}, and the oscillators’ natural frequencies \citep{smith2019chaos}. All these bring the necessity of developing new generative models for multimodal time series data.



In this paper, we develop a novel time-series generator, named \emph{directed chain GANs} (DC-GANs), motivated by the formulation of DC-SDEs \citep{detering2020directed}. The drift and diffusion coefficients in DC-SDEs depend on another stochastic process, which we call the neighborhood process, with distribution required to be the same as the SDEs' distribution. Different from other GANs, which only use real data in discriminators, our proposed algorithm naturally takes the dataset as the neighborhood process, giving generators access to data information. This feature enables our model to outperform the state-of-the-art methods on many datasets, particularly for the situation of multimodal time-series data.

\paragraph{Contribution.} We propose a generator for multimodal distributed time series based on DC-SDEs (cf. Definition~\ref{def:dc_sde}), and prove that our model can handle any distribution that Neural SDEs are capable of generating (see Theorem~\ref{thm:flexibility}). To train the generator, we propose to use a combination of two types of discriminators: Sig-WGAN \citep{ni2021sig} and Neural CDEs \citep{kidger2020neural}.
    
    We notice that data generated immediately from DC-GANs can be correlated, and propose an easy solution by walking along the directed chain in the path space for further steps (see Theorem~\ref{thm:dependence_decay}). Combining branching the chain with different Brownian noises  enables our model to generate unlimited independent fake data.
    
    We test our algorithms in four different experiments and show that DC-GANs provide the best performance compared to existing popular models, including SigWGAN \citep{ni2021sig}, CTFP \citep{deng2020modeling}, Neural SDEs \citep{kidger2021neural}, TimeGAN \citep{yoon2019time} and Transformer-based generator TTS-GAN \citep{li2022tts}.

\paragraph{Related Literature.}
Neural ordinary differential equations (Neural ODEs), introduced by \citet{chen2018neural}, use neural networks to parameterize the vector fields of ODEs and bring a powerful tool for learning time series data. Later, significant effort has been put into improving Neural ODEs, e.g., \citet{quaglino2019snode,zhang2019anodev2,massaroli2020dissecting,hanshu2019robustness}. In fact, incorporating mathematical concepts into the Neural ODEs framework can provide the capability of analyzing and justifying its validity, leading to a deeper understanding of the framework itself. For example, \citet{li2020scalable} and \citet{tzen2019neural} generalized the idea to neural stochastic differential equations (Neural SDEs), providing adjoint equations for efficient training. By integrating rough path theory \citep{lyons2007differential}, \citet{kidger2020neural} proposed neural controlled differential equations (Neural CDEs) {and \citet{morrill2021neural} proposed neural rough differential equations} for modeling time series. Other examples integrating profound mathematical concepts include using higher order kernel mean embeddings to capture information filtration \citep{salvi2021higher}, and solving high dimensional partial differential equations through backward stochastic differential equations \citep{han2018solving}, to name a few.

The closely related model to ours is the Neural SDEs by \citet{kidger2021neural}, which uses the Wasserstein GAN method to train stochastic diffusion evolving in a hidden space and gains great success in simulating time series data. Other successful GANs models for time-series data include \citet{cuchiero2020generative,tzen2019theoretical,deng2020modeling,kidger2021neural,li2022tts}; see \citet{brophy2022generative} for a recent review. Note that we find in the numerical experiments that the performances of Neural SDEs are limited in simulating multimodal distributed time series, e.g., as shown in Figure~\ref{fig:opinion} from the stochastic opinion dynamics (Example 1 in Section~\ref{sec:OpinionDyn}).


The directed chain is one of the simplest structures in random graph theory, where each node on the graph represents a stochastic process and has interactions only with its neighbor nodes (Figure~\ref{fig:chain}). To our best knowledge, \citet{detering2020directed} initiated the study of the SDE system on the directed chains, followed by \citet{feng2020linear, feng2020linear0} for the analysis of stochastic differential games on such chains with (deterministic and random) interactions. 
Later on, more complicated graph structures are studied beyond directed chains. For example, \citet{lacker2021locally} analyzed particle behaviors where the interaction only happens between neighborhoods in an undirected graph, and proved Markov random fields property and constructed Gibbs measure on path space when interactions appear only in drift; \citet{lacker2022case} considered stochastic differential games on transitive graphs; \citet{carmona2022stochastic} studied games on a graphon which has infinitely many nodes. Despite numerous extensions, {we find that the directed chain structure, although simple but rich enough for generating multimodal time series.} 

From another viewpoint, DC-SDEs can be understood as the reverse direction of mimicking theorems \citep{gyongy1986mimicking}. The idea of ``mimicking" is that for a general SDE (even with path-dependence features), one can construct a Markovian one to mimic its marginal distribution; see \citet{brunick2013mimicking} for details on mimicking aspects of It\^{o} processes including the distributions of running maxima and running integrals. DC-SDEs work in the reverse direction: they can produce marginal distributions that are generated by Markovian SDEs (see Theorem~\ref{thm:flexibility} for a detailed statement). The benefit of using DC-SDEs, in particular in machine learning, is to have a more vital fitting ability by embedding data into a slightly more complicated system.


\begin{figure*}[ht]
\centering
\subfloat[$t = 0.1$]{\includegraphics[width=0.15\textwidth]{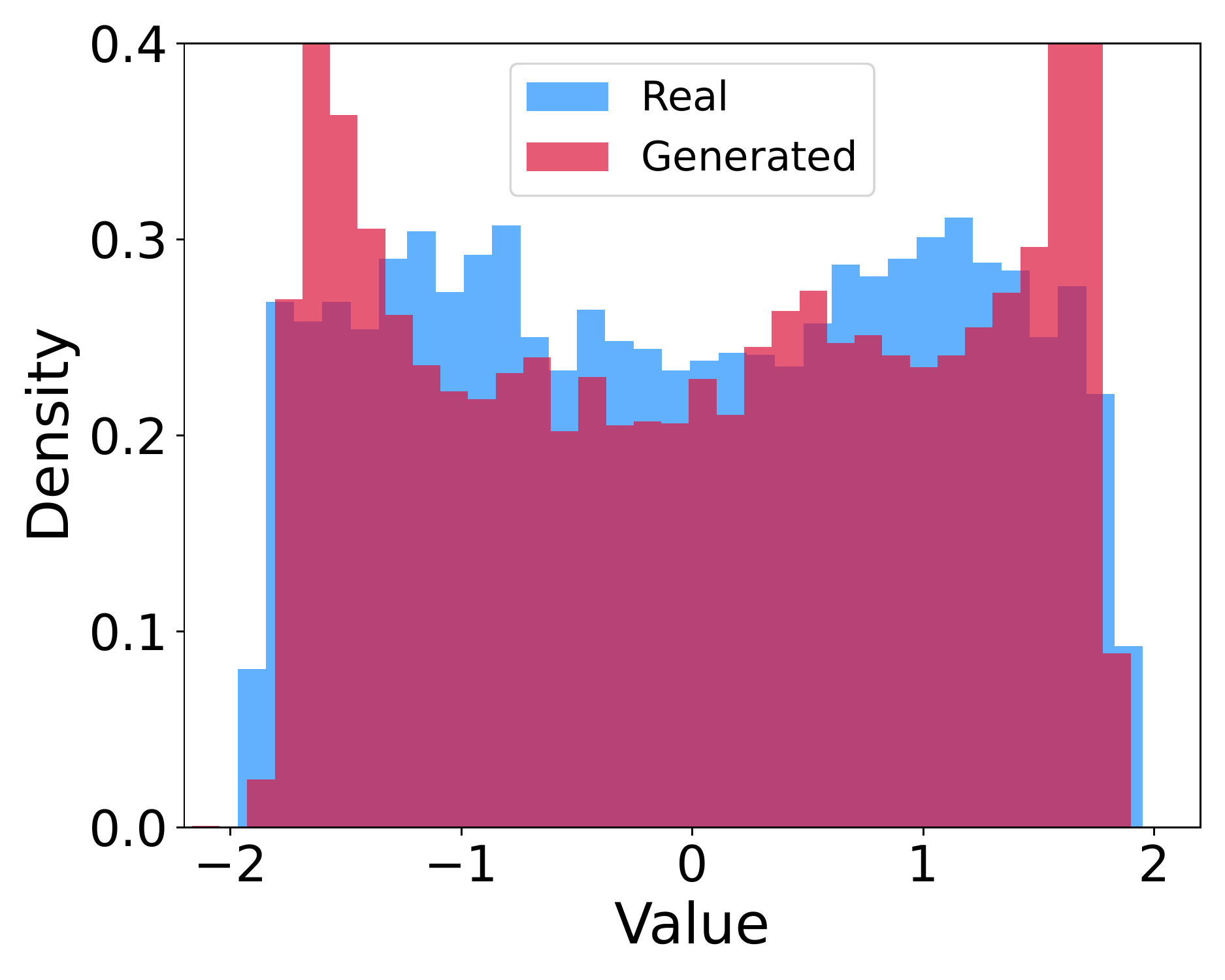}} 
\subfloat[$t = 0.3$]{\includegraphics[width=0.15\textwidth]{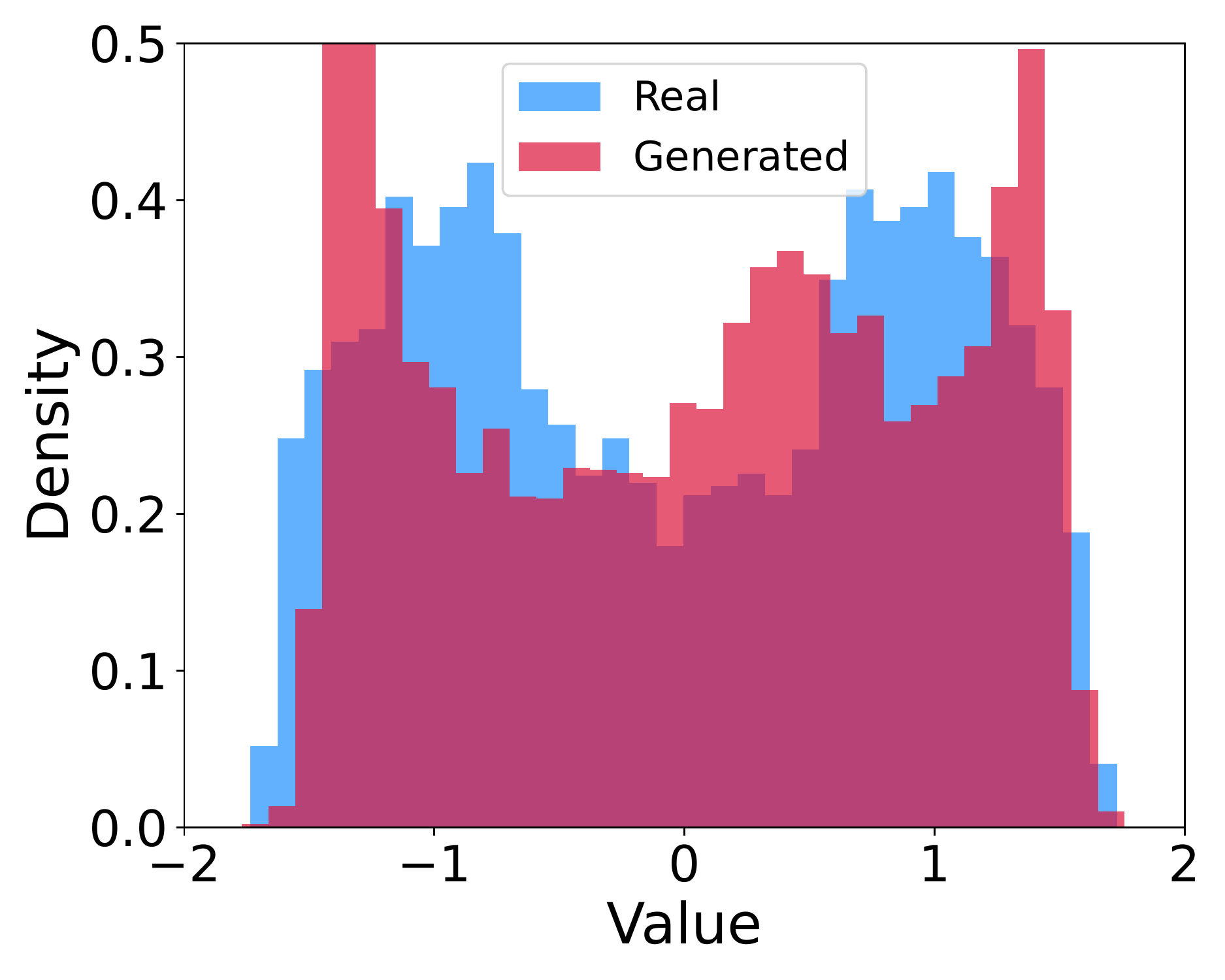}}
\subfloat[$t = 0.5$]{\includegraphics[width=0.15\textwidth]{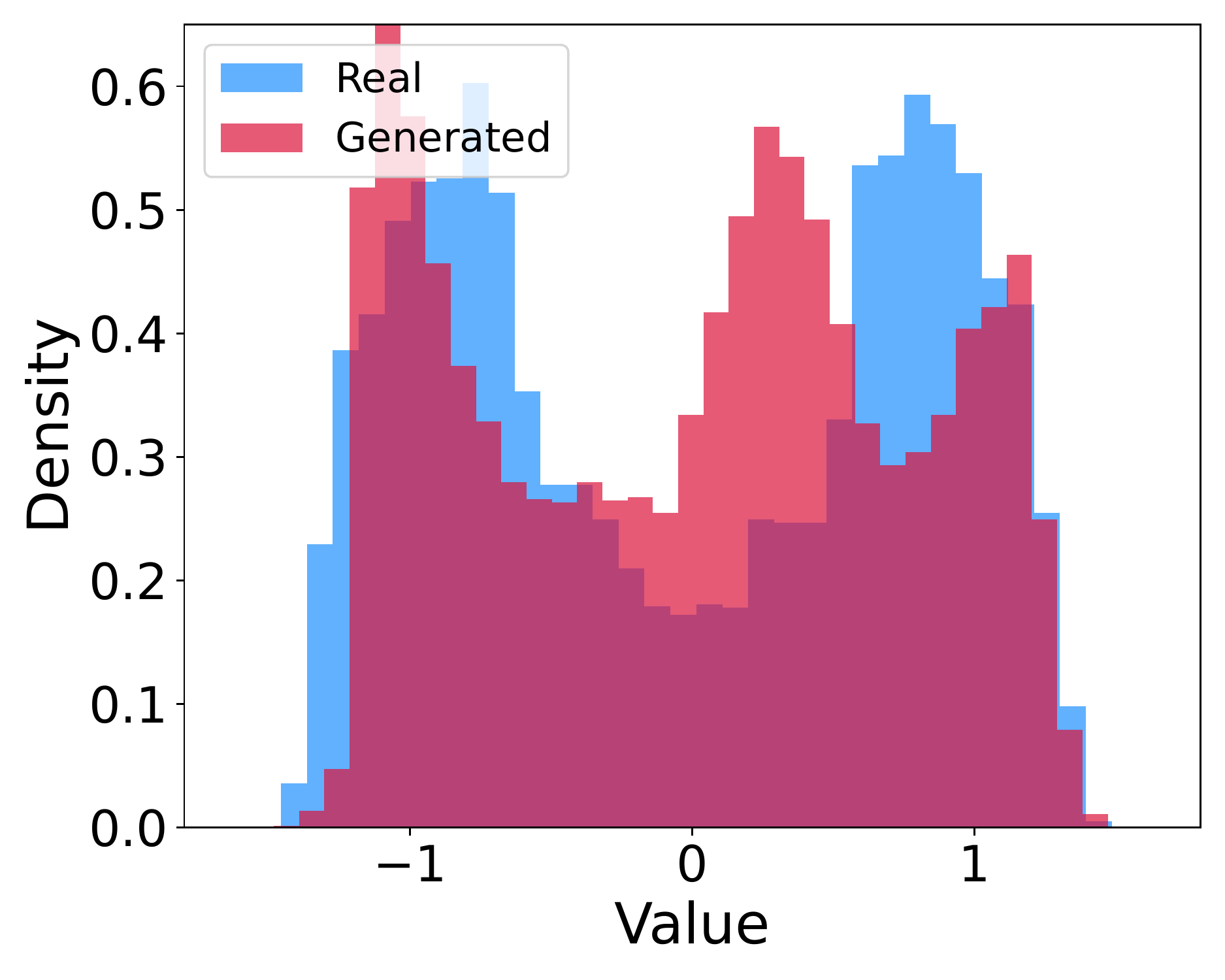}} 
\subfloat[$t=0.7$]{\includegraphics[width=0.15\textwidth]{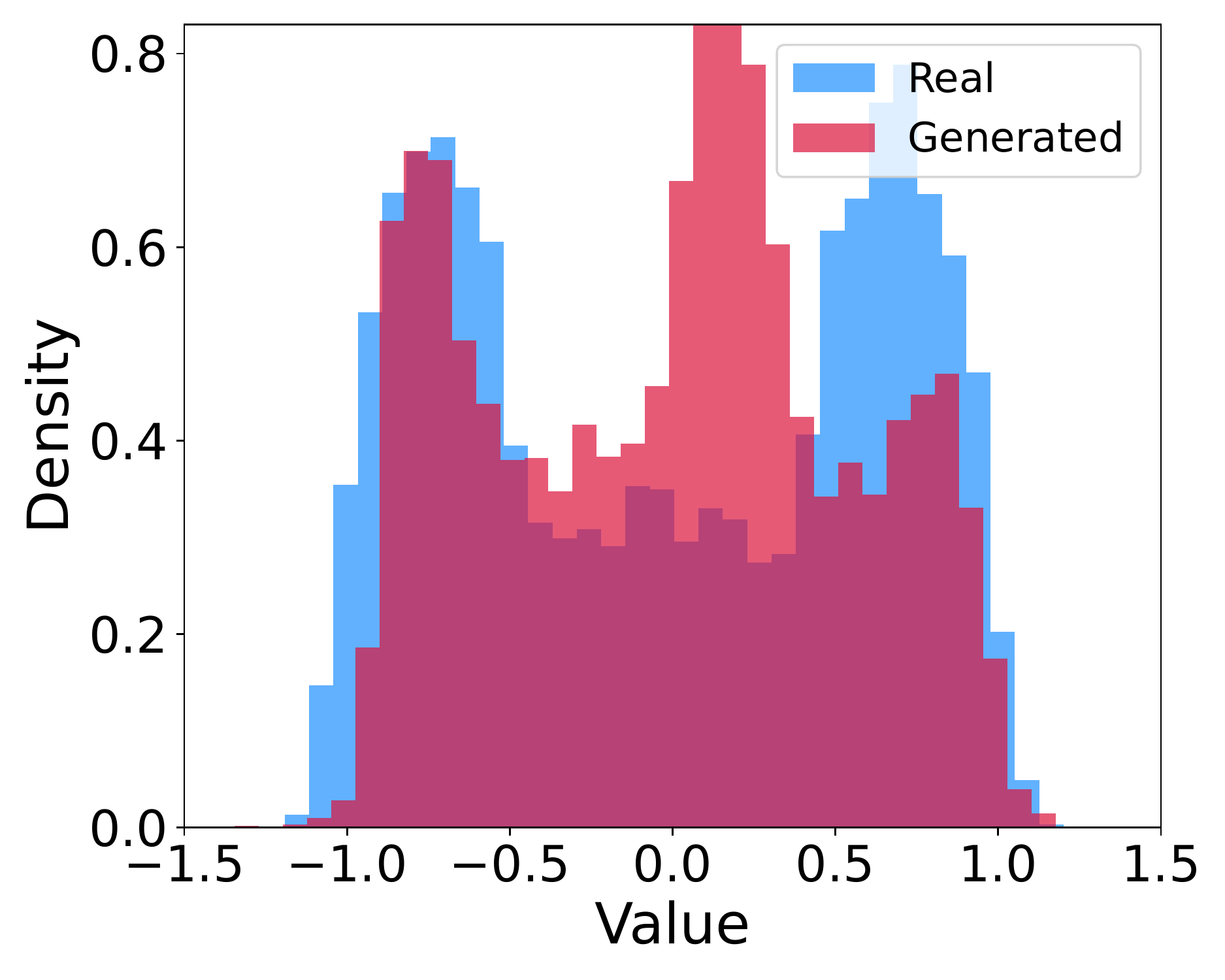}}
\subfloat[$t=0.9$]{\includegraphics[width=0.15\textwidth]{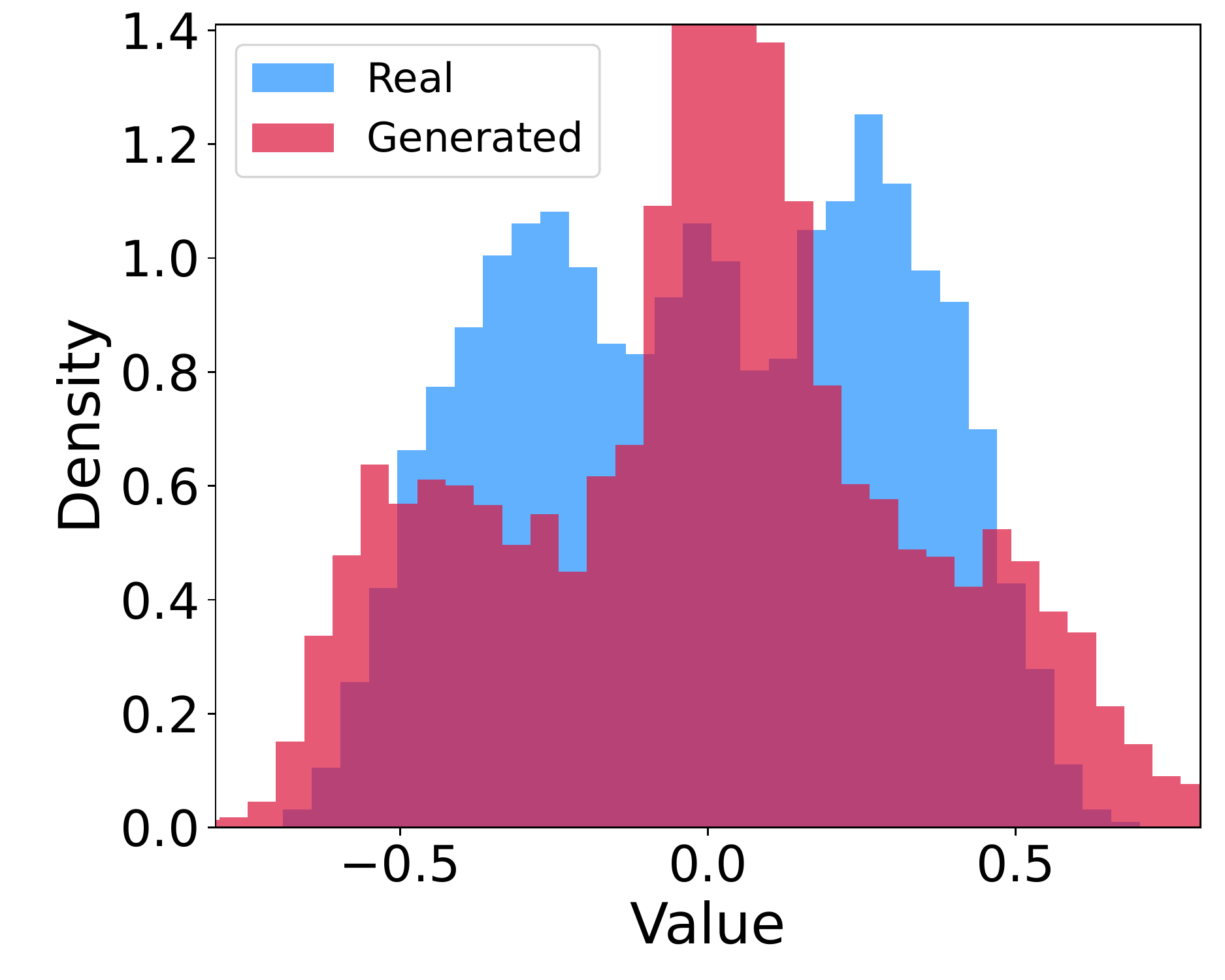}} 
\subfloat[$t=1.0$]{\includegraphics[width=0.15\textwidth]{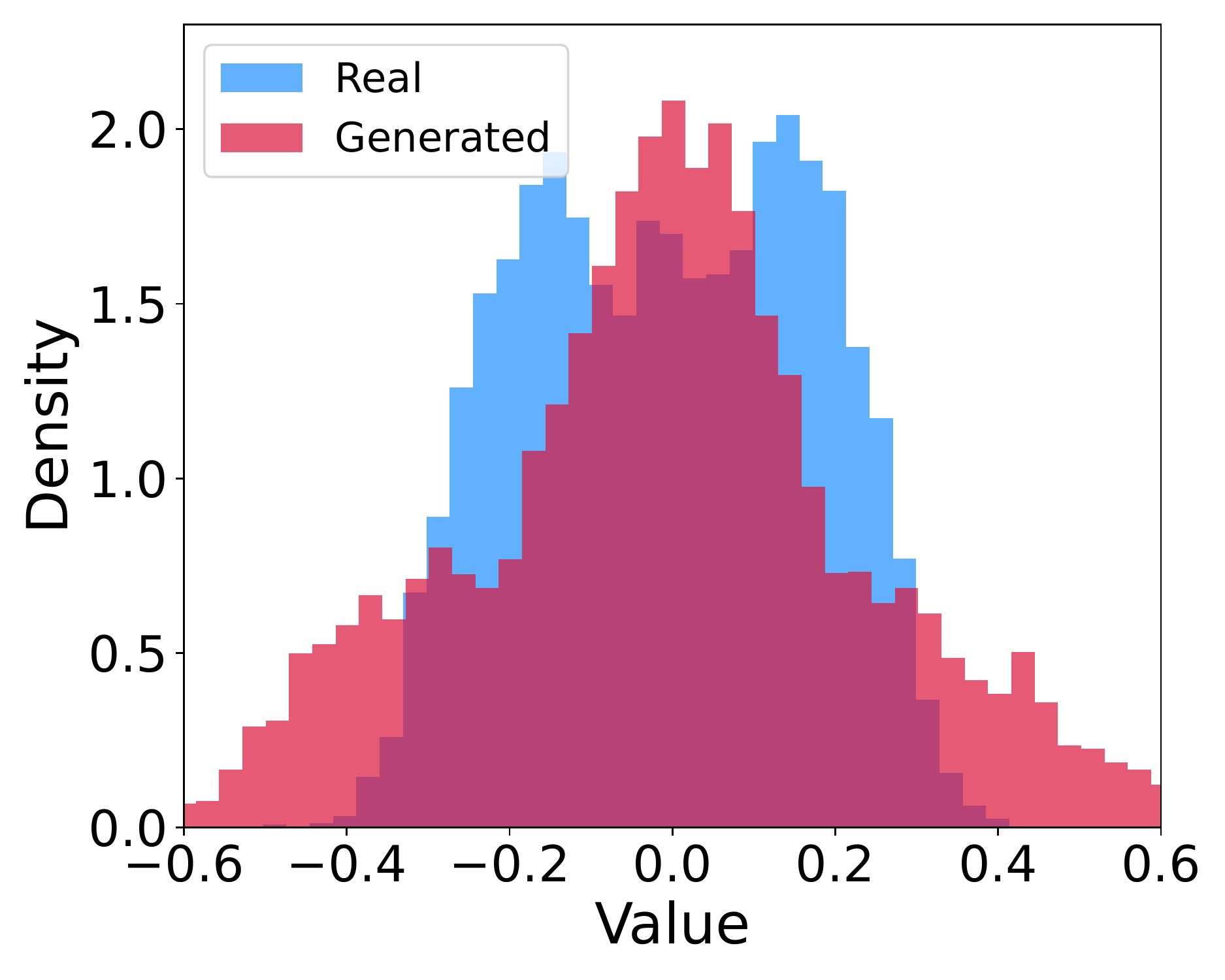}} \\
\subfloat[$t=0.1$]{\includegraphics[width=0.15\textwidth]{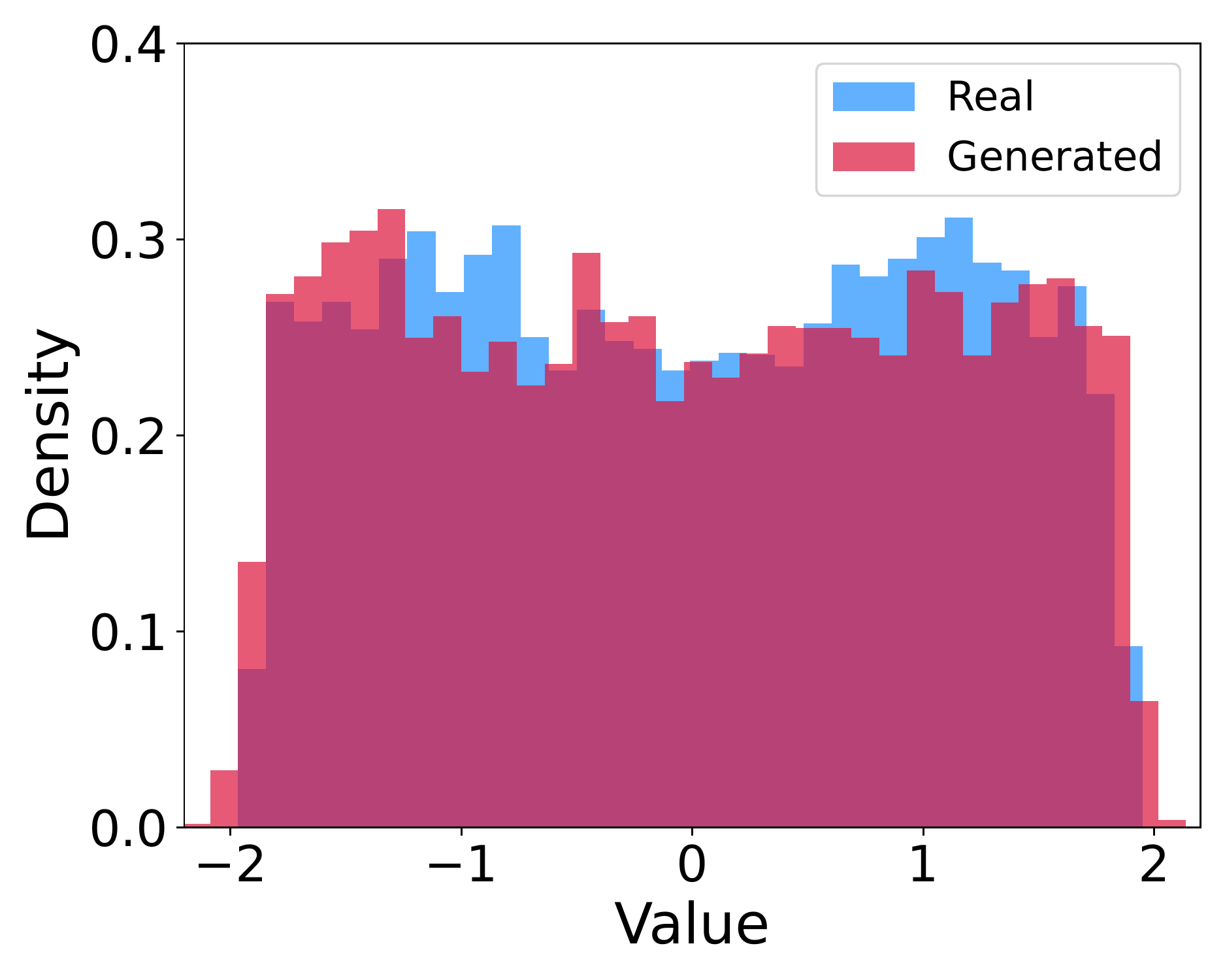}} 
\subfloat[$t=0.3$]{\includegraphics[width=0.15\textwidth]{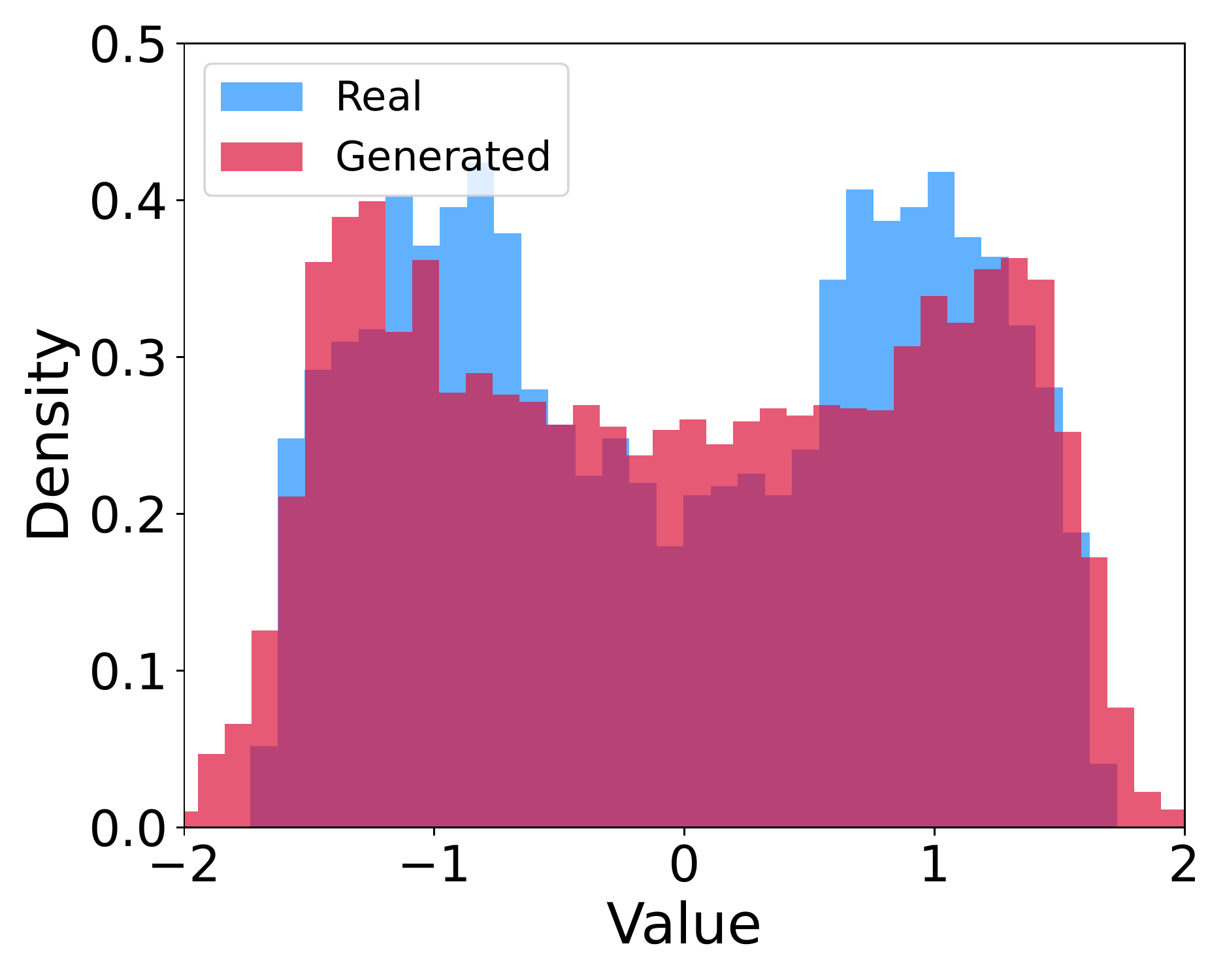}}
\subfloat[$t=0.5$]{\includegraphics[width=0.15\textwidth]{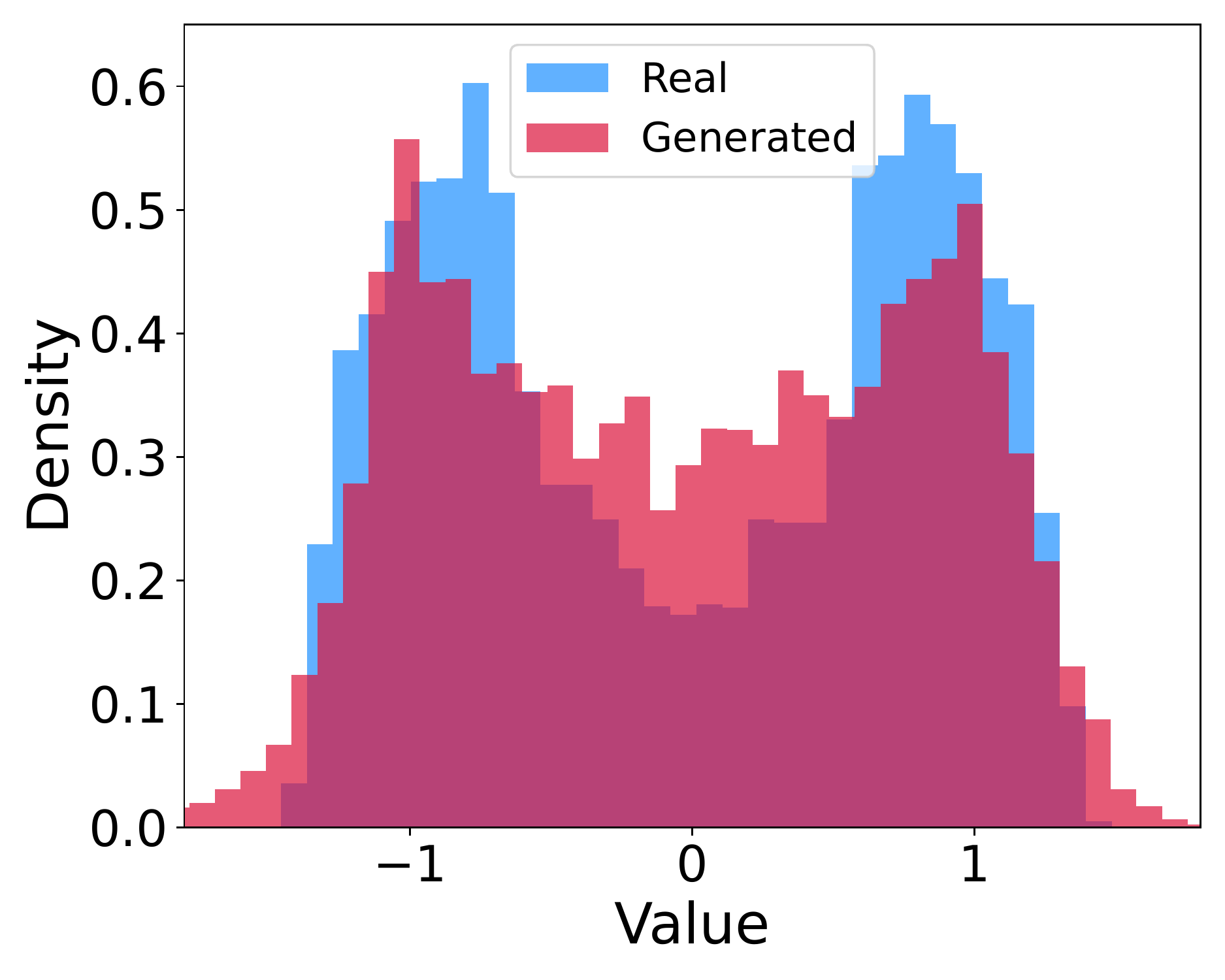}} 
\subfloat[$t=0.7$]{\includegraphics[width=0.15\textwidth]{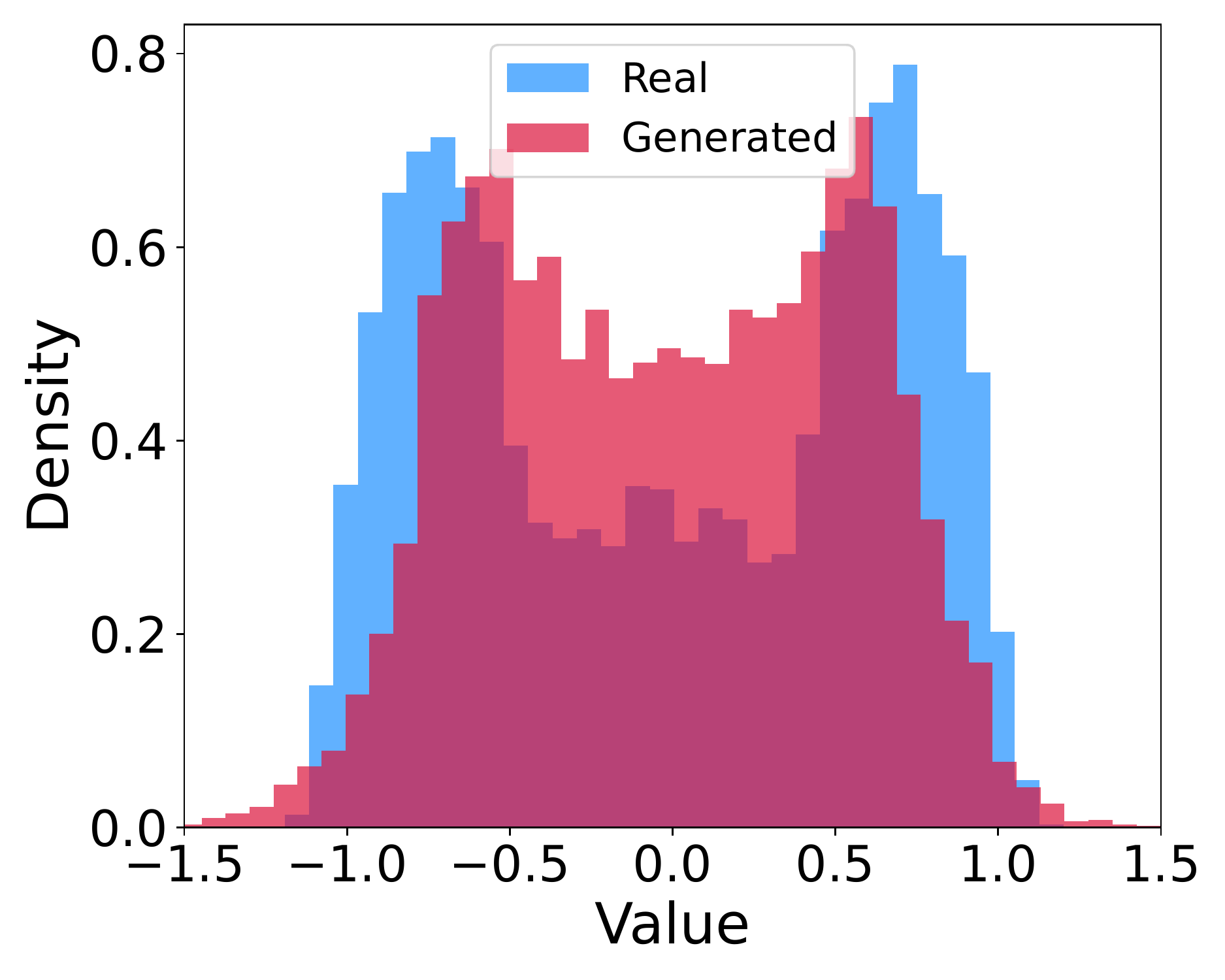}}
\subfloat[$t=0.9$]{\includegraphics[width=0.15\textwidth]{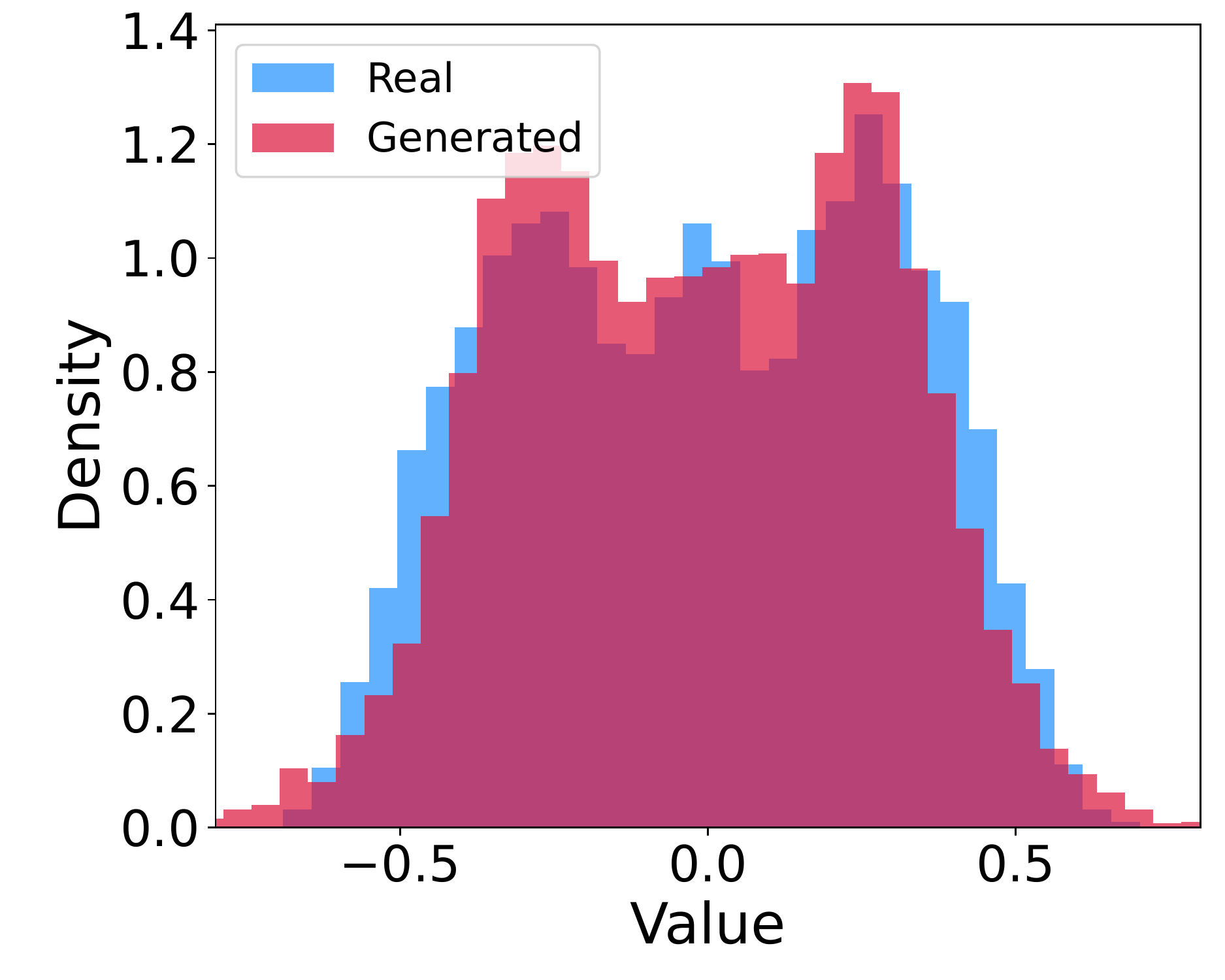}} 
\subfloat[$t=1.0$]{\includegraphics[width=0.15\textwidth]{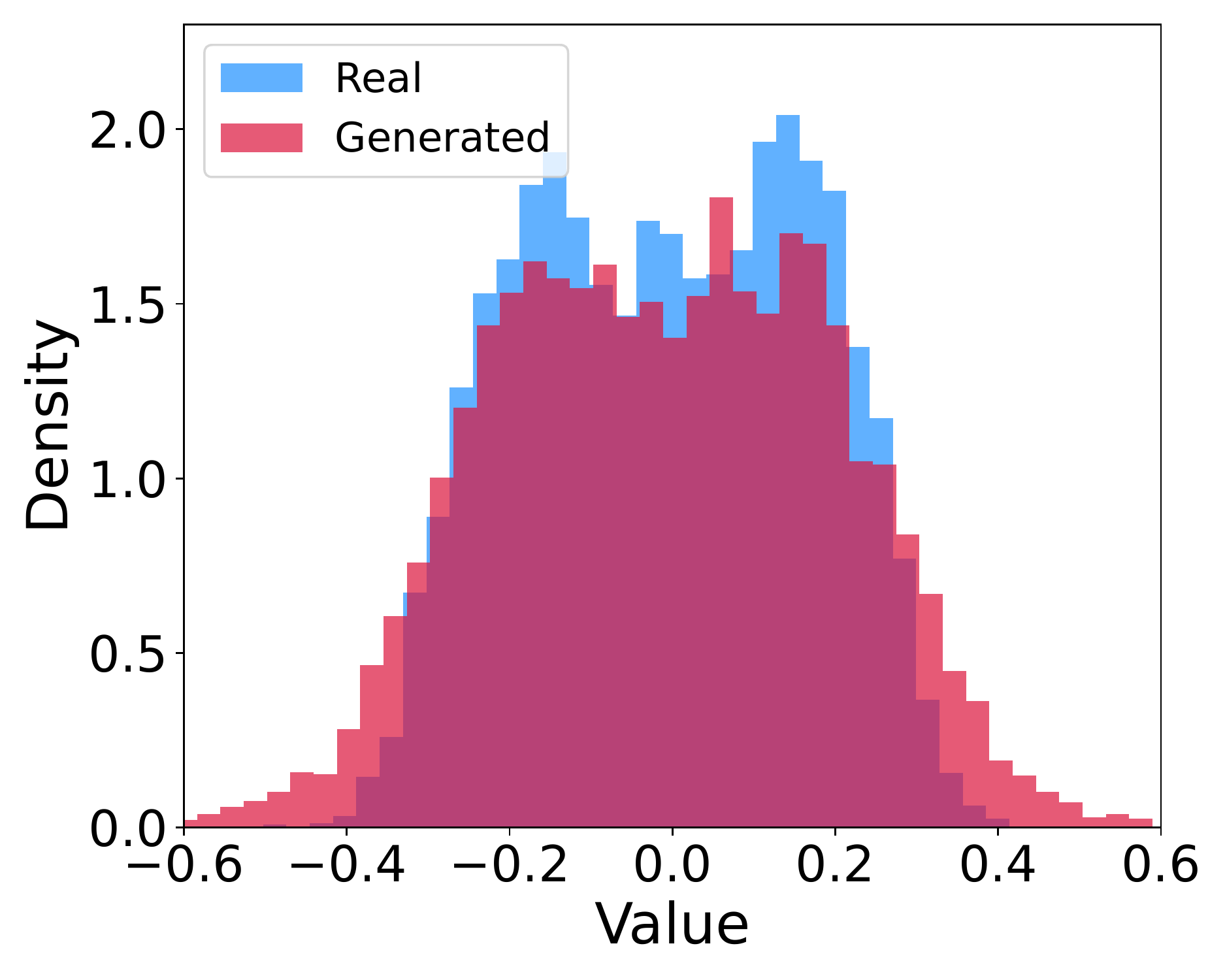}}
\vspace{-.1in}
\caption{Marginal distributions of real data (blue) and generated data (red) from Example 1 (Stochastic opinion dynamics) at $t\in\{0.1, 0.3, 0.5, 0.7, 0.9, 1\}$ in Section~\ref{sec:OpinionDyn}. Figures (a)--(f) are generated by Neural SDEs, and Figures (g)--(l) are generated by DC-GANs. One can see from Figures (e) and (f) that Neural SDEs fail to capture the bimodal distribution.}
\label{fig:opinion}

\end{figure*}

\section{DIRECTED CHAIN SDEs AND SIGNATURES}
In this section, we introduce two mathematical concepts that serve as the backbones of our algorithm: directed chain SDEs and signatures. {In Section~\ref{sec:dc-sde}, we identify the central issue of naively generating time series from true data using DC-SDEs: the non-independence of the true data and fake data (Problem~\ref{prob:independence}). Then we overcome the non-independence issue by Decoorelating and Branching Phase in Section~\ref{sec:generator}, and provide theoretical guarantees for this procedure (Theorem~\ref{thm:dependence_decay}).}

In the sequel, we shall use $X_s, X_t$ to denote the state of $X$ at time $s$ and $t$, respectively. With no subscript, e.g., by $X$, we mean the whole path from $t=0$ to $T$.

\subsection{Directed Chain SDEs (DC-SDEs)}\label{sec:dc-sde}


{The DC-SDEs are the limit of a system of $n$-coupled SDEs interacting homogeneously on a directed chain when $n$ goes to infinity. Below we will focus on DC-SDEs and defer the introduction of this limiting process to Appendix~\ref{app:prelim}.}
Under the general setup, DC-SDEs can be of McKean-Vlasov type where the coefficients have distributions as inputs, corresponding to the $n$-coupled system having mean-field interaction. In our proposed generator, it is sufficient to use the simple case mentioned above, DC-SDE without the mean-field interaction, as in the following definition. 

\begin{defn}[DC-SDEs]
\label{def:dc_sde}
Fix a filtered probability space $(\Omega, \mc{F}, (\mc{F}_t)_{t \ge 0}, \PP)$ and a finite time horizon $[0, T]$.
Let $(X, \tilde{X})$ with $X, \tilde{X}\in L^2(\Omega\times [0,T], \RR^N)$ be a pair of square-integrable stochastic processes satisfying
\begin{equation} \label{eq:dc_sde}
    X_t = \xi \hspace{-2pt} + \hspace{-3pt}\int_0^t V_0(s, X_s, \tilde{X}_s)\ud s +  \hspace{-3pt}\int_0^t V_1(s, X_s, \tilde{X}_s)\ud B_s,
\end{equation}
for $t\in [0,T]$, with the {distributional constraint}
\begin{equation} \label{eq:dc_sde_a}
    \Law(X_t, 0\le t\le T)= \Law(\tilde{X}_t, 0\le t\le T),
\end{equation}
where {Law($\cdot$)} {stands for the distribution}, $V_0\in \RR^N$ and $V_1\in\RR^{N\times d}$ are smooth coefficients satisfying Lipschitz and linear growth conditions, $B$ is a standard $d$-dimensional Brownian motion, and $X_0:=\xi$, $\tilde{X}$ and $B$ are assumed to be independent.
\end{defn}


The existence of the solution to  \eqref{eq:dc_sde} and the weak uniqueness in the sense of distribution have been proved under the Lipschitz and linear growth assumptions on the coefficients in \citet{detering2020directed} for a simple case, and in \citet{ichiba2022smoothness} for a more general case. Moreover, with the smoothness of the solution under certain additional conditions posed on the coefficients (cf. \citet{ichiba2022smoothness}), we can derive a partial differential equation (PDE) for the marginal densities of the solution. 
Then, the associated PDEs lead to the following theorem: DC-SDEs have at least the same amount of flexibility as Neural SDEs. 
\begin{thm}
\label{thm:flexibility}
Under proper assumptions, for any $Y$ that satisfies a system of Markovian SDEs on $[0,T]$, there exists a unique solution to the DC-SDE \eqref{eq:dc_sde} with constraints \eqref{eq:dc_sde_a}, some $V_0$ and non-degenerate coefficients $V_1$, such that they have the same marginal distributions for all $t\in[0,T]$. 
Here by degenerate, we mean that $V_i(t, x, \tilde{x}):=V_i(t, x),\, i\in\{0,1\}$, i.e., the coefficients have no dependence on neighborhood nodes at all. 
\end{thm}
We defer the proof of Theorem~\ref{thm:flexibility} to Appendix~\ref{app:flexibility}.



Naturally, if $V_0$ and $V_1$ are known (or learned from data), one can take real data paths as $\tilde X$ in \eqref{eq:dc_sde} and straightforwardly generate paths of $X$ that have the same distribution as $\tilde X$ by the constraint \eqref{eq:dc_sde_a}. However, naively implementing this idea will lead to the following potential problems. 


\begin{problem}[Lack of Independence]
\label{prob:independence}
%
    The distribution of the generated sequence crucially depends on the real data; Consequently, to avoid dependence, {a single real path can only be used once as $\tilde X$} to generate one path of $X$, and thus the number of the generated sequence has to be the same as that of the training data set in one run. 
\end{problem}
Note that a qualified generator should also be able to generate \emph{unlimited} independent data that does not depend on the original one. 
Fortunately, both problems mentioned above can be overcome by the idea behind the following theorem.
\begin{thm}
\label{thm:dependence_decay}
Under mild non-degeneracy conditions, the correlation between training data and generated data in DC-SDEs decays exponentially fast, as the distance increases on the chain. 
\end{thm}
Due to the page limit, we give the formal statement of Theorem~\ref{thm:dependence_decay} with detailed proof in Appendix~\ref{app:dependence-decay}.

We shall explain how to beat the independence problem during the implementation described in Section~\ref{sec:generator}. As shown in Appendix \ref{app:dependence-decay}, the introduction of independent Brownian motions to \eqref{eq:dc_sde} is the key to solving the independence problem. We shall also provide an extreme example (cf. Remark \ref{rem:degenerate_eg}) showing that without $\int V_1 \ud B$, the system \eqref{eq:dc_sde}--\eqref{eq:dc_sde_a} has only trivial (deterministic) solution. 

\subsection{Signature} \label{sec:SignatureDef}
The proposed method utilizes signature \citep{lyons2007differential}, a concept from rough path theory that we shall briefly introduce for completeness. As an infinitely graded sequence, the signature can be understood as a feature extraction technique for time series data with certain regularity conditions. Let $x: \Omega\times[0,T] \to \RR^N$ be a continuous random process, and denote the signature map by $S: x\mapsto S(x) \in T(\RR^N)$, where $T(\RR^N)$ is the tensor algebra defined on $\RR^N$. Then,
$$S(x):=(1, x^1,\cdots, x^i, \dots), \quad \text{and}$$
$$x^i = \int_{0<t_1<\cdots < t_i < T} \ud x_{t_1} \otimes\cdots \otimes \ud x_{t_i}.$$
Signature characterizes paths uniquely up to the tree-like equivalence, and the equivalence is removed if at least one dimension of the path is strictly increasing \citep{boedihardjo2016signature}. The concept has been applied to design machine learning methods, e.g., \citet{chevyrev2018signature, kidger2019deep, min2021signatured, ni2021sig, min2022convolutional,dyer2021approximate}. In practice, one needs to truncate the signature up to a finite order $M$, denoted by $S^M(x)=(1,x^1,\cdots,x^M)$. We will justify the truncation by the factorial decay property and discuss other theoretical properties of signature in Appendix \ref{app:sig_decay}.  

In addition, signature induces another powerful tool to characterize the distribution of random processes: the \emph{expected signatures}. It was proved by \citet{chevyrev2016characteristic} that expected signatures characterize the distribution of random processes uniquely, i.e., if $\EE[S(x)]=\EE[S(y)]$ and $\EE[S(x)]$ has an infinite radius of convergence, then $x$ and $y$ have the same distribution. 


\section{PROPOSED METHOD: DC-GANs}
\label{sec:model}

In this section, we describe DC-GANs for generating multimodal distributed time series. Our method builds on the DC-SDEs with a straightforward idea: To find the (sub-) optimal solution of the generator, we implement a GAN model with the Neural DC-SDEs as the generator. For the discriminator, we use Neural CDEs \citep{kidger2021neural} and Sig-Wasserstein GAN \citep{ni2020conditional, ni2021sig}.

\subsection{Generator}\label{sec:generator}
To overcome the independence issue explained in Problem~\ref{prob:independence}, we design DC-GANs by two phases: 1) training and 2)  decorrelating and branching. The second phase will be utilized during testing. Both $V_0$ and $V_1$ in \eqref{eq:dc_sde} will be parameterized by multi-layers fully connected NNs. 

{\bf Training Phase.}
We set aside the independence problem and focus on finding the optimal coefficients $V_0$ and $V_1$ (together with the discriminator). Denote the training data by $\{\Tilde{X}(\omega_i)\}_{i=1}^M$, where each $\omega_i$ represents a realization of the randomness in the path space. We treat our training data $\{\Tilde{X}(\omega_i)\}_{i=1}^M$ as the neighborhood process $\Tilde{X}$ in \eqref{eq:dc_sde}. For each training path data $\Tilde{X}(\omega_i)$, we generate a DC-SDE path $X(\omega_i)$, according to the Euler scheme of \eqref{eq:dc_sde},
\begin{align}
    &X_{t_{j+1}}(\omega_i) \nonumber\\
    & = X_{t_j}(\omega_i) + V_0(t_j, X_{t_j}(\omega_i), \tilde X_{t_j}(\omega_i)) (t_{j+1} - t_j) \label{eq:dc_sde-simulation}\\
    & \hspace{12pt}+ V_1(t_j, X_{t_j}(\omega_i), \tilde X_{t_j}(\omega_i))(B_{t_{j+1}}(\omega_i)-B_j(\omega_i)), \nonumber
\end{align}
where $0 = t_0 \leq t_1 \leq \ldots \leq t_J = T$ is a partition on $[0,T]$, $\{B(\omega_i)\}_{i=1}^M$ are independent Brownian paths. 

Both the generated paths $\{X(\omega_i)\}_{i=1}^M$ and the training paths $\{\Tilde{X}(\omega_i)\}_{i=1}^M$ will be passed into the discriminator, where their Wasserstein distance needs to be minimized. 
To simplify the notations for later use, we define $G_\theta: (\xi, B, \Tilde{X}) \mapsto X$ to represent the overall transformation in \eqref{eq:dc_sde-simulation}, with $\theta$ denoting all network parameters of $V_0$ and $V_1$.


{\bf Decorrelating and Branching Phase.}
During testing, we utilize a branching scheme to alleviate the independence problem; see Figure~\ref{fig:branch} for an illustrative example.
Let $q$ be the number of steps we ``walk'' along the directed chain. Here ``walking'' along the chain means: After we have finished the training (identified $V_0$ and $V_1$) phase, we start with the first chain (the grey one in Figure \ref{fig:branch}). We take real data as the first neighborhood $X_1$ to generate $X_2$ through the scheme \eqref{eq:dc_sde-simulation}, where $X_1$ takes the role of $\tilde X$ and $X_2$ takes the role of $X$. Then we use $X_2$ as the neighborhood to generate $X_3$, and repeat this procedure until we obtain $X_q$. By Theorem \ref{thm:dependence_decay}, $X_q$ and $X_1$ are asymptotically uncorrelated, as $q \to \infty$. We describe the pseudo-code in Algorithm~\ref{algo:testing} below for this decorrelating step. 

\begin{figure}[htpb]
    \centering
    \vspace{.05in}
    \includegraphics[width=0.4\textwidth]{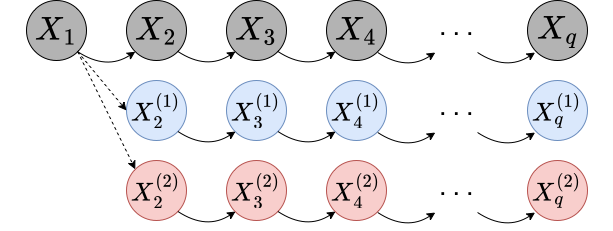}
    \vspace{.1in}
    \caption{Branching Scheme. Let $q$ be the number of steps we ``walk'' along the directed chain. We take real data as the first neighborhood $X_1$ to generate $X_2$ through the scheme \eqref{eq:dc_sde-simulation}, where $X_1$ takes the role of $\tilde X$ and $X_2$ takes the role of $X$. Then we use $X_2$ as the neighborhood to generate $X_3$, and repeat this procedure until we obtain $X_q$. }
    \label{fig:branch}
\end{figure}

\begin{algorithm}
\caption{Generator in the Decorrelating and Branching}
\label{algo:testing}
\begin{algorithmic}
\STATE {\bfseries Input:} real data $\{\Tilde{X}(\omega_i)\}_{i=1}^M$, \# of steps  $q$, generator $G_\theta$;
\STATE {\bfseries Set} $\{X_{1}(\omega_i)\}_{i=1}^M :=\{\Tilde{X}(\omega_i)\}_{i=1}^M$;
\FOR{$k=2$ {\bfseries to} $q$}
\STATE Generate $M$ independent copies of initials positions and Brownian paths $\{\xi_k(\omega_i), B_k(\omega_i)\}_{i=1}^M$;
\STATE Generate $M$ paths $\{X_k(\omega_i)\}_{i=1}^M$ by $$X_k(\omega_i) = G_\theta(\xi_k(\omega_i), B_k(\omega_i), X_{k-1}(\omega_i));$$
\ENDFOR
\STATE {\bfseries Output:} $\{X_q(\omega_i)\}_{i=1}^M$ 
\end{algorithmic}
\end{algorithm}

To generate more fake data, we can initiate more chains with the same starting node $X_1$  (where the real data are) and independent Brownian paths, and then ``walk'' along the chain to get $X_q^{(i)}, i=2,3,\dots$. Again, $X_q^{(i)}$ is asymptotically uncorrelated to $X_1$. By the definition of DC-SDEs, $X_q$ and $X_q^{(i)}$ are conditionally independent with conditioning on $X_1$. Therefore, we can claim that $X_q$ and $X_q^{(i)}$ are asymptotically uncorrelated. 


{\bf Architecture.} Note that although the directed chain SDE pair $(X, \Tilde{X})$ is Markovian, $X$ itself can be non-Markovian as a standalone stochastic process. All the historical information can be embedded in the neighborhood process and fetched through $V_0$ and $V_1$. Such a property leads to one of the key differences between our method and Neural SDEs: there is no need to embed time series into a hidden space. In our implementations, $V_0$ and $V_1$ take standard feedforward neural networks; see Appendix~\ref{app:experiment} for details.

\subsection{Discriminator}
The purpose of the discriminator is to identify the optimal parameters in the $V_0$- and $V_1$- networks. We use the Wasserstein GAN  framework \citep{goodfellow2020generative, arjovsky2017wasserstein} to train the generator, and two types of discriminators will be used here.

{\bf SigWGAN.} 
Using the idea of expected signature, \citet{ni2020conditional, ni2021sig} designed  Sig-Wasserstein GAN by directly minimizing the signature Wasserstein-1 distance,
\begin{equation*}
\mathrm{Sig\mbox{-}W}_1(\mu, \nu):= |\EE_{X\sim\mu}[S(X)] - \EE_{X\sim\nu}[S(X)]|,
\end{equation*}
where $\mu$ and $\nu$ are two distributions of time series corresponding to real data and fake data, $S$ is the signature map, and $|\cdot|$ is the $l_2$ norm. For practical use,  we approximate the infinite sequence $S$ by truncating signatures up to some finite order $m$, i.e.,
\begin{equation}
\label{eq:sig_mmd}
\mathrm{Sig\mbox{-}W}^m_1(\mu,\nu):= |\EE_{\mu}[S^m(X)] - \EE_{\nu}[S^m(X)]|.
\end{equation}
The higher the truncation order $m$, the more information the signature can capture. However, the number of terms in the truncated signature will grow exponentially and become costly when the time series data is high-dimensional.

{\bf Neural CDEs.} Neural controlled differential equations are the second candidate for the discriminator when the underlying time series is of high dimension. This is also the discriminator used in \cite{kidger2021neural}. Let $D_\phi: X \mapsto R$ be a Neural CDE discriminator where $\phi$ denotes the network parameters. The training goal is to solve the following optimization problem for the generator
$$\min_{\theta} \EE_{\xi, B}\big[D_\phi\big(G_\theta(\xi, B, \Tilde{X})\big)\big],$$
and the following one for the discriminator
\begin{align}
\max_\phi \big\{ \EE_{\xi, B}\big[ D_\phi\big(G_\theta(\xi, B, \Tilde{X})\big)\big] - \EE_{\Tilde{X}}\big[D_\phi(\Tilde{X}) \big]\big\}. \label{eq:cde_wgan}
\end{align}
Compared to only using the Neural CDEs as the discriminator, we notice that a combination of Neural CDEs and lower-order signature Wasserstein-1 distance as the discriminator works better for the third numerical example below. That is, the generator is optimized with respect to
\begin{align}
\min_{\theta} \big\{ &\EE_{\xi, B}\big[D_\phi\big(G_\theta(\xi, B, \Tilde{X})\big)\big] \nonumber\\
& + \mathrm{Sig\mbox{-}W}^m_1\big(\Law(\Tilde{X}), \Law\big(G_\theta(\xi, B, \Tilde{X})\big)\big) \big\}.\label{eq:combined_loss}
\end{align}
Remark that DC-GANs can work with different discriminators, and here we choose to use neural CDEs and SigWGAN as the discriminators.
The pseudo-algorithm of the overall training strategy is summarized in Algorithm~\ref{algo:training}. 


\begin{algorithm}
\caption{The Training Phase}
\label{algo:training}
\begin{algorithmic}
\STATE {\bfseries Input:} real data $\{\Tilde{X}(\omega_i)\}_{i=1}^M$, boolean variable $cde$, total epochs $E$, signature truncation order $m$;
\FOR{$e=1$ {\bfseries to} $E$}
\STATE Generate independent copies of initials and Brownian motions $(\xi(\omega_i), B(\omega_i))_{i=1}^M$;
\STATE Generate fake data $\{X(\omega_i)\}_{i=1}^M$ by 
$$X(\omega_i) = G_\theta(\xi(\omega_i), B(\omega_i), \Tilde{X}(\omega_i));$$\vspace{-10pt}
\IF{$cde$ is True}
\STATE Compute the loss \eqref{eq:cde_wgan} and its gradients w.r.t. $\phi$;
\STATE Compute the loss \eqref{eq:combined_loss} and its gradients w.r.t. $\theta$;
\STATE Update $\theta$ by stochastic gradient descent optimiser;
\STATE Update $\phi$ by stochastic gradient ascent optimiser;
\ELSE
\STATE Compute the loss \eqref{eq:sig_mmd} and its gradients w.r.t. $\theta$;
\STATE Update $\theta$ by stochastic gradient descent optimiser;
\ENDIF
\ENDFOR
\STATE {\bfseries Output:} Generator $G_\theta$.
\end{algorithmic}
\end{algorithm}

\section{EXPERIMENTS}
We present the performance of the proposed DC-GANs on four different datasets, including stochastic opinion dynamics, network dynamics from neural science, and real-world stock data and energy consumption data. In all cases, we set $q=10$, i.e., ``walk'' along the chain for ten steps during the decorrelating phase. Other hyperparameters for neural network training can be found in Appendix~\ref{app:experiment} for details. The example implementation of DC-GAN is available at \url{https://github.com/mmin0/DirectedChainSDE}.

{\bf Benchmarks \& Evaluation.} The first two synthetic datasets are generated by SDEs, the third real-world data set of stock price time series was extracted from Yahoo Finance\footnote{\url{https://finance.yahoo.com/quote/GOOG?p=GOOG&.tsrc=fin-srch}.}, and the fourth real-world energy consumption data were obtained from Ireland's open data portal\footnote{\url{https://data.gov.ie/dataset?theme=Energy}.}. 
We compare our results by DC-GANs with SigWGAN, CTFP, Neural SDEs and Transformer-based TTS-GANs, and DC-GANs give much better accuracy under discriminative, predictive, and maximum mean discrepancy (MMD) metrics detailed below. We also provide independence metrics to show that our decorrelating and branching scheme can resolve the independence problem. {We also test over different discriminators, and show the flexibility of choosing the one that brings better performance or has a faster running time.}



\subsection{Metrics}

{\bf Marginal Distribution \& MMD.} 
For the first two examples, we plot histograms to compare their marginal distributions at several time stamps. To measure the goodness of fitting for time series, we use {maximum mean discrepancy} (MMD) induced by the expected signature given in \eqref{eq:sig_mmd}.

{\bf Discriminative Metric.}
To quantitively measure the similarity between the fake data generated by DC-GANs and real data, we train a post-hoc time series classifier by optimizing a two-layer LSTM to discriminate original and fake sequential data. The fake data is labeled \textit{nonreal} and the original data is labeled \textit{real}. The worse discriminative ability of the post-hoc time series classifier implies the better performance of the time series generator. Our discriminative score is calculated as the absolute difference between 0.5 and predicting accuracy on testing data, thus a smaller score indicates a better generator.

{\bf Predictive Metric.}
Typically, a useful time series dataset contains temporal evolution information, and we can predict the future given past data. We expect that DC-GANs can capture this temporal dynamic property accurately from the original data. To this end, we train an auxiliary two-layer LSTM sequential predictor on the generated time series and test this post-hoc predictor on the original time series. The predictive score is calculated as the $L^1$ distance between predicted sequences and true sequences on testing data (the real data), with smaller scores for better generators.

{\bf Independence Metric.}
It is crucial for success to show that our algorithm can address the independence problem. As an independence metric, we use 
\begin{equation}
\label{eq:independence_score}
\rho(x, y) := \sup_{t\in[0,T]} \| \rho(x_t, y_t) \|_1,
\end{equation}
where $x, y\in L^2(\Omega\times[0,T], \RR^N)$ and $\rho(x_t, y_t)$ represents the cross-correlation matrix between random vectors $x_t, y_t$. Smaller $\rho(x, y)$ means less correlation between real data $x$ and generated data $y$. 

These metric scores are used to measure the algorithms' discriminative ability, predictive ability, and the distance and correlation between the generated data and the true distribution. A decrease in the metric score suggests an improvement in the quality of the generated data. Furthermore, the decrease in independence score is supported by Theorem~\ref{thm:dependence_decay} and Theorem~\ref{thm:dependence_decay_formal} in Appendix.

All experiments are run over ten different random seeds, and we report the mean and standard deviation (in the parentheses) for all metrics in Tables~\ref{tab:opinion}--\ref{tab:energy}. We give more details on how all these metrics are implemented in Appendix \ref{app:experiment_metrics}.

\subsection{Example 1: Stochastic Opinion Dynamics}\label{sec:OpinionDyn}

We first consider stochastic opinion dynamics modeled by the following MV-SDE
\begin{equation*}
    \ud Y_t = -\bigg[ \int_\RR \varphi_\theta(\|Y_t - y\|)(Y_t-y) \;\mu_t(\mathrm{d} y) \bigg]\ud t + \sigma \ud W_t,
\end{equation*}
where $\varphi_\theta$ is a interaction kernel with $\theta_1, \theta_2 >0$,
\begin{align*}
    \varphi_\theta(r) = \left\{ \begin{array}{ll}
    \theta_1 \exp\bigg( -\frac{0.01}{1-(r-\theta_2)^2} \bigg),     & r>0, \\
       0,  & r\le 0,
    \end{array} \right.
\end{align*}
and $\mu_t = \text{Law}(Y_t)$ denotes the distribution of $Y_t$.  
One can interpret $\theta_1$ as a scale parameter that characterizes the intensity of the attraction between entities, and $\theta_2$ as the range parameter that determines the distance, within which an entity must be of one another in order to interact. This model is widely used in many disciplines, from flocking and swarming behaviors in biology (where $Y_t$ is the position) to public opinion evolution in social science (where $Y_t$ is the opinion towards a topic). We refer to \citet{motsch2014heterophilious} for further details. 

We choose $\theta_1 =6$, $\theta_2 =0.2$, $\sigma = 0.1$, $T = 1$, $\Delta t = 0.01$, and generate 8192 paths. The distribution $\mu_t$ is approximated by the empirical distribution of 8192 samples. These samples are used to produce the blue density in Figure~\ref{fig:opinion}, where a clear shift in distribution from unimodality to bimodality is observed. 



We first compare with the Neural SDEs method \citep{kidger2021neural}. Figure~\ref{fig:opinion} gives the comparison of the marginal distributions at $t = 0.1, 0.3, 0.5, 0.7, 0.9, 1.0$. One can see that DC-GANs can accurately capture the bimodal distribution in general, but the Neural SDE method can not. Under the MMD metric \eqref{eq:sig_mmd}, the discrepancy of DC-GANs is $0.07$, while the Neural SDEs give $0.12$. More comparisons with SigWGAN, CTFP, Neural SDEs and TTS-GAN under discriminative, MMD, and independence metrics are provided in Table~\ref{tab:opinion}. Our proposed DC-GANs have a smaller discriminative score, and an independence score comparable with the ones produced by the Neural SDE generator, SigWGAN, CTFP and TTS-GAN, all of which generate purely independent samples. Therefore, we conclude that DC-GANs can produce fake data closer to the real data without independence issues.



\begin{table*}[ht]
\caption{Stochastic Opinion Dynamics (Example 1). The scores are computed for SigWGAN, CTFP, Neural SDEs, TTS-GAN and DC-GANs under different metrics. The numbers in the parenthesis are the corresponding standard deviations of each score. Note that a smaller value means a better approximation, which indicates the DC-GANs provide more accurate fake data with compared independence and running time.} \label{tab:opinion}
\vskip 0.15in
\begin{center}
\begin{small}
\begin{sc}
\begin{tabular}{cllll}
\toprule
Method  & Discriminative & MMD & Independence & Time (min) \\
\midrule
SigWGAN & 0.213 (0.01)   & 0.328 (0.004) & 0.009(0.004) & 6.55\\
CTFP &  0.131 (0.02)  & 0.281 (0.005) & 0.010(0.003) & 5.58 \\
Neural SDEs             &  0.045 (0.025) & 0.122 (0.003) & 0.007 (0.005) & 7.07\\
TTS-GAN & 0.127(0.014) & 0.176(0.003) & 0.008(0.003) & 15.6 \\
DC-GANs            & {\bf 0.028 (0.019)} & {\bf 0.07 (0.003)} & 0.009 (0.004) & 6.82   \\
\bottomrule
\end{tabular}
\end{sc}
\end{small}
\end{center}
\vskip -0.1in
\end{table*}

\subsection{Example 2: Stochastic FitzHugh-Nagumo Model}

FitzHugh-Nagumo model is a standard model from neuroscience \citep{baladron2012mean,reisinger2022adaptive}, used to describe the neurons' interacting spiking. Mathematically, for $N$ neurons and $P$ different neuron populations, and $i\in\{1,\dots,N\}$,  we denote by $p(i)=\alpha, \alpha\in\{1,\dots,P\}$ the population of $i$-th particle that belongs to. The state vector of neural $i$, $(X_t^{i, N})_{t\in[0,T]}=(V_t^{i,N}, w_t^{i,N}, y_t^{i,N})_{t\in[0,T]}$, satisfies the SDE, 
\begin{equation*}
\begin{aligned}
    \ud X_t^{t, N} &= f_\alpha(t, X_t^{t, N})\ud t + g_\alpha(t, X_t^{t, N})\left[ \begin{array}{c}
         \ud W_t^i  \\
         \ud W_t^{i,y} 
    \end{array} \right] \\
    & \quad+ \sum_{\gamma=1}^P\frac{1}{N_\gamma}\sum_{j,p(j)=\gamma}\bigg( b_{\alpha\gamma}(X_t^{i,N}, X_t^{j,N})\ud t\\
    &\qquad + \beta_{\alpha\gamma}(X_t^{i,N}, X_t^{j,N})\ud W_t^{i,\gamma} \bigg),
\end{aligned}
\end{equation*}
where $V$ denotes a short, nonlinear elevation of membrane voltage, $w$ denotes a slower, linear recovery variable,
$N_\gamma$ denotes the number of neurons in the population $\gamma$. We defer more details about model description and  training data generation to Appendix \ref{app:experiment_details}. 

The FitzHugh-Nagumo system is an example of a relaxation oscillator, and exhibits a characteristic excursion in phase space, before the variables $V$ and $w$ relax back to their rest values. As a result, their distributions are typically multimodal distributed; see Figure~\ref{fig:fitzhugh-dist} in Appendix~\ref{app:experiment_details}.



Figure \ref{fig:fitzhugh} depicts the differences of their joint marginal densities between generated time series and training (real) time series on channels $1$ and $3$ at $t=0.1, 0.3, 0.5, 0.7, 0.9, 1.0$. The darker the color the smaller the differences, thus the closer the distribution and indicating a better generator. It can be observed that DC-GANs produce less difference in joint marginal densities at multiple time stamps. Under discriminative, predictive, and MMD metrics, DC-GANs give better samples than SigWGAN, CTFP, Neural SDEs, and TTS-GAN consistently; see Table \ref{tab:fitzhugh}. In particular, fake samples produced by DC-GANs are almost indistinguishable for a two-layer LSTM classifier after exhaustive training. By the comparison using MMD, one can see that DC-GANs generate fake samples with distributions significantly closer to real data than the other three methods. The independence scores given by \eqref{eq:independence_score} are nearly indistinguishable.


\begin{figure*}[ht]
\vspace{-.05in}
\centering
\subfloat[$t=0.1$]{\includegraphics[width=0.15\textwidth]{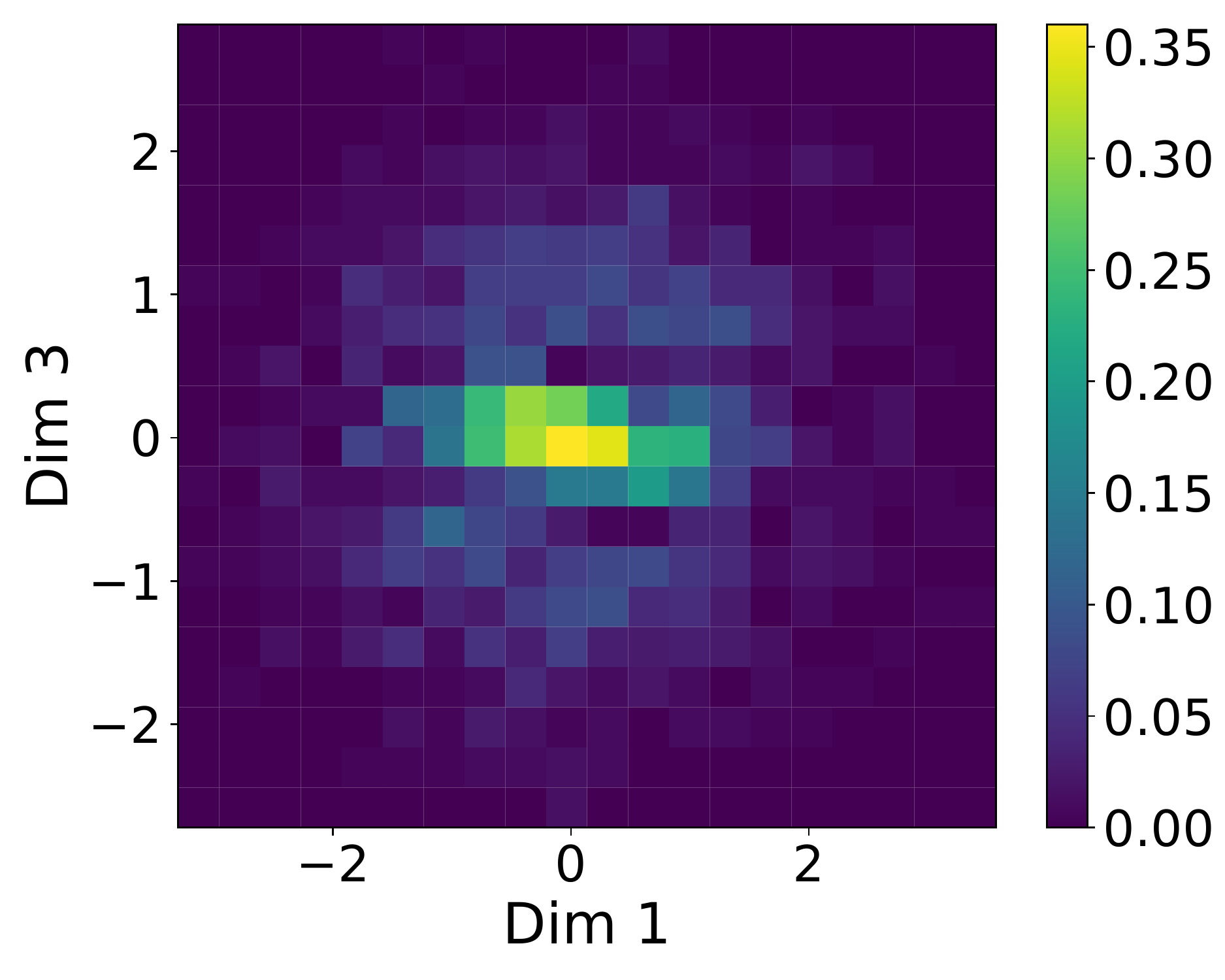}} 
\subfloat[$t=0.3$]{\includegraphics[width=0.15\textwidth]{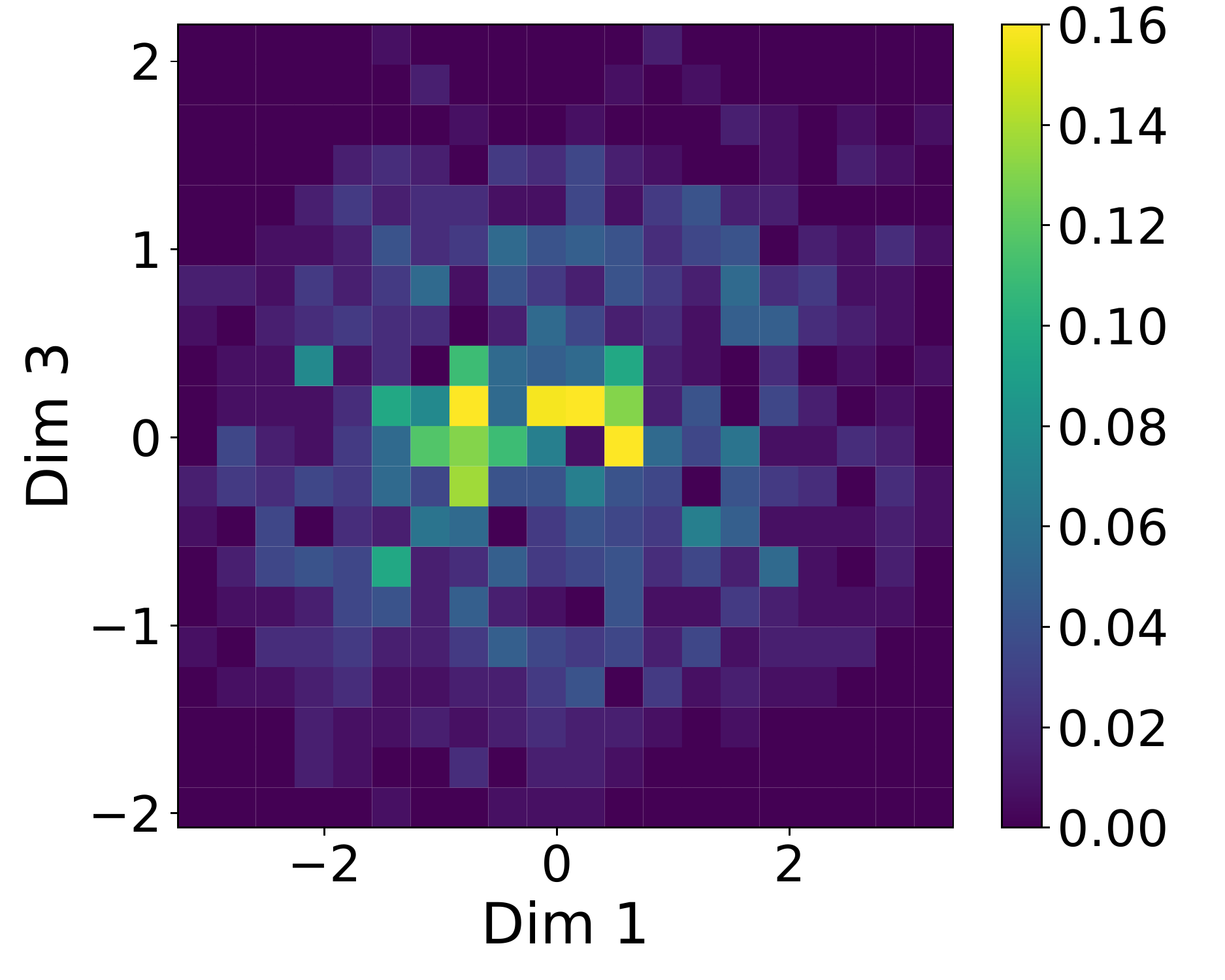}}
\subfloat[$t=0.5$]{\includegraphics[width=0.15\textwidth]{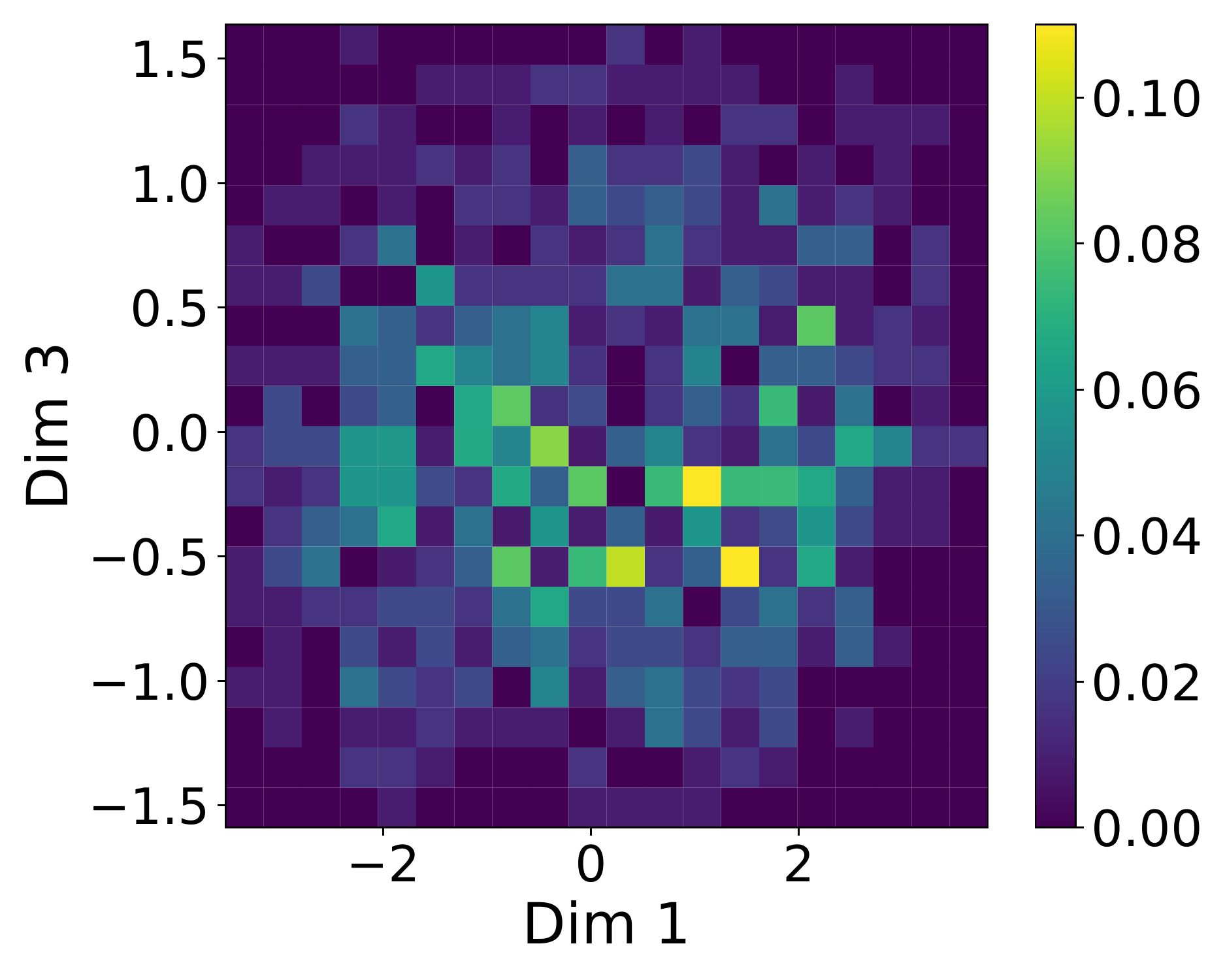}} 
\subfloat[$t=0.7$]{\includegraphics[width=0.15\textwidth]{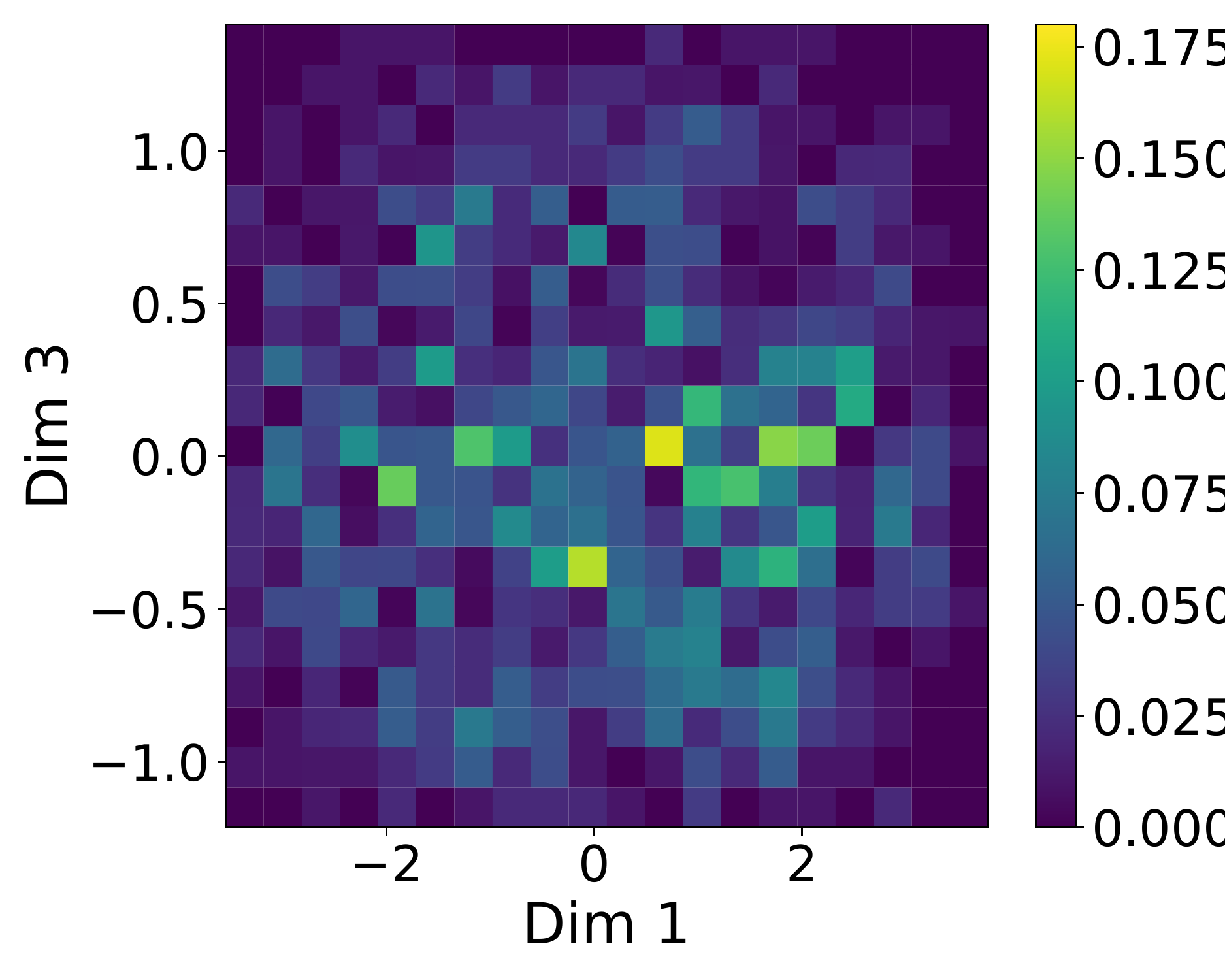}}
\subfloat[$t=0.9$]{\includegraphics[width=0.15\textwidth]{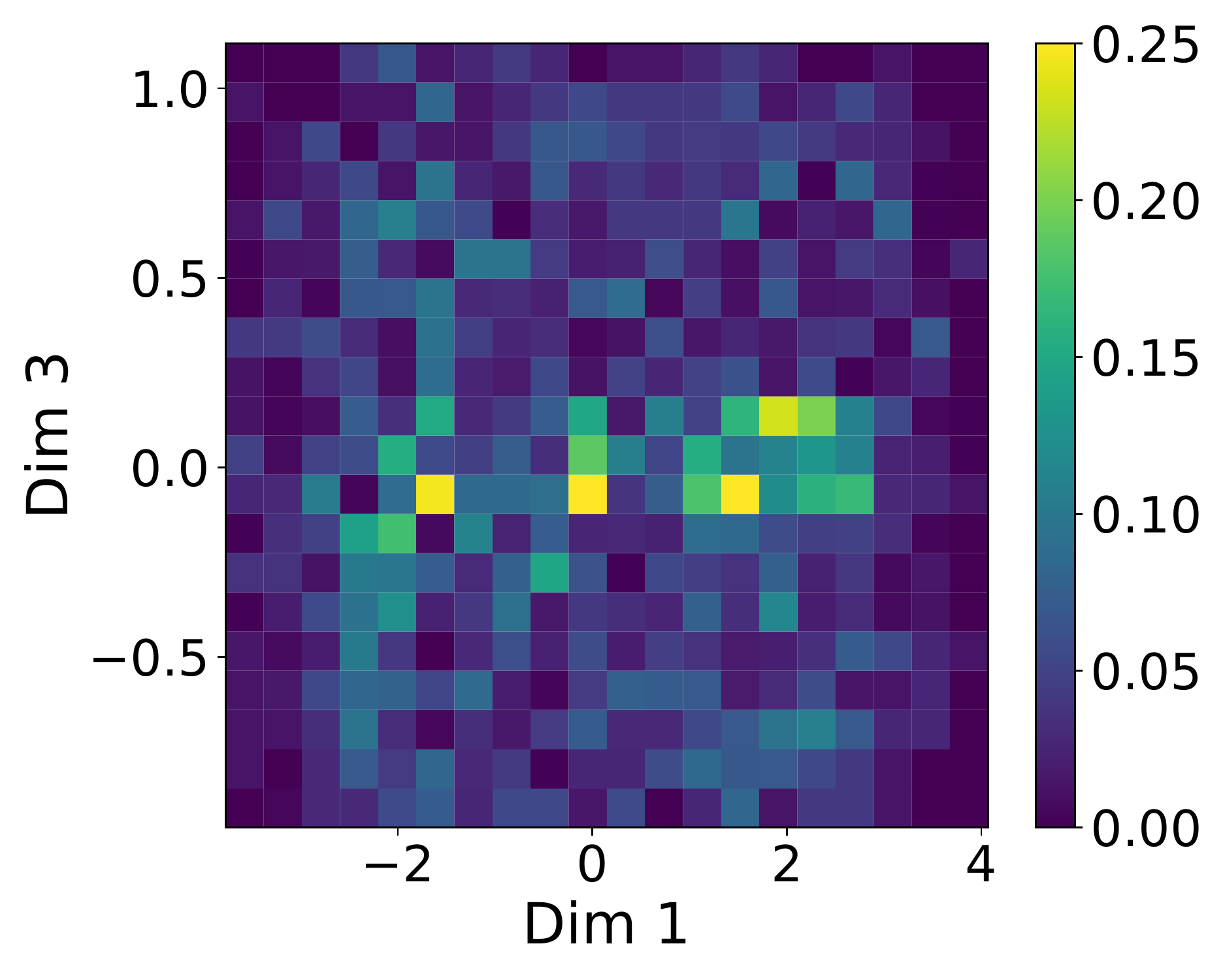}} 
\subfloat[$t=1.0$]{\includegraphics[width=0.15\textwidth]{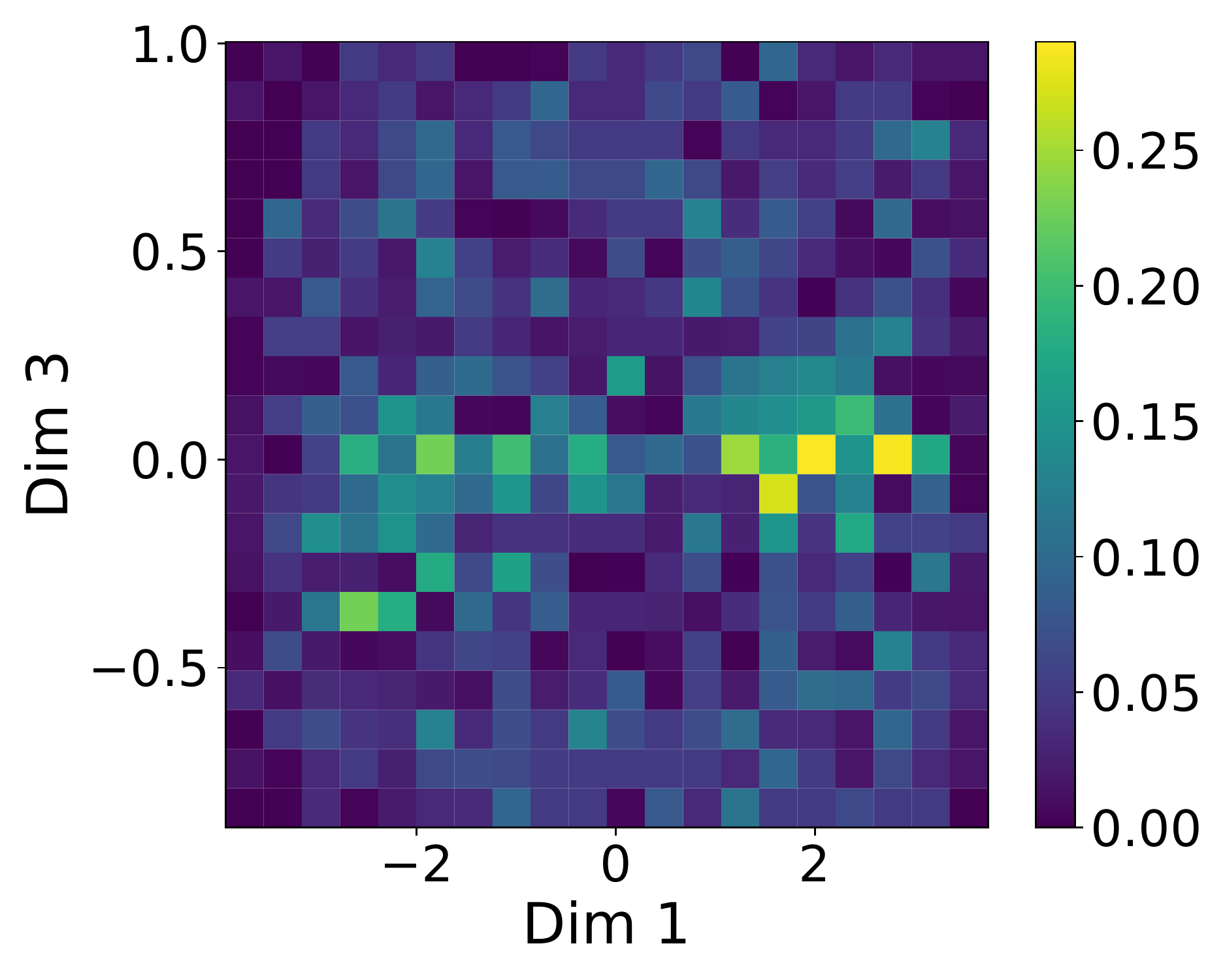}} \\
\subfloat[$t=0.1$]{\includegraphics[width=0.15\textwidth]{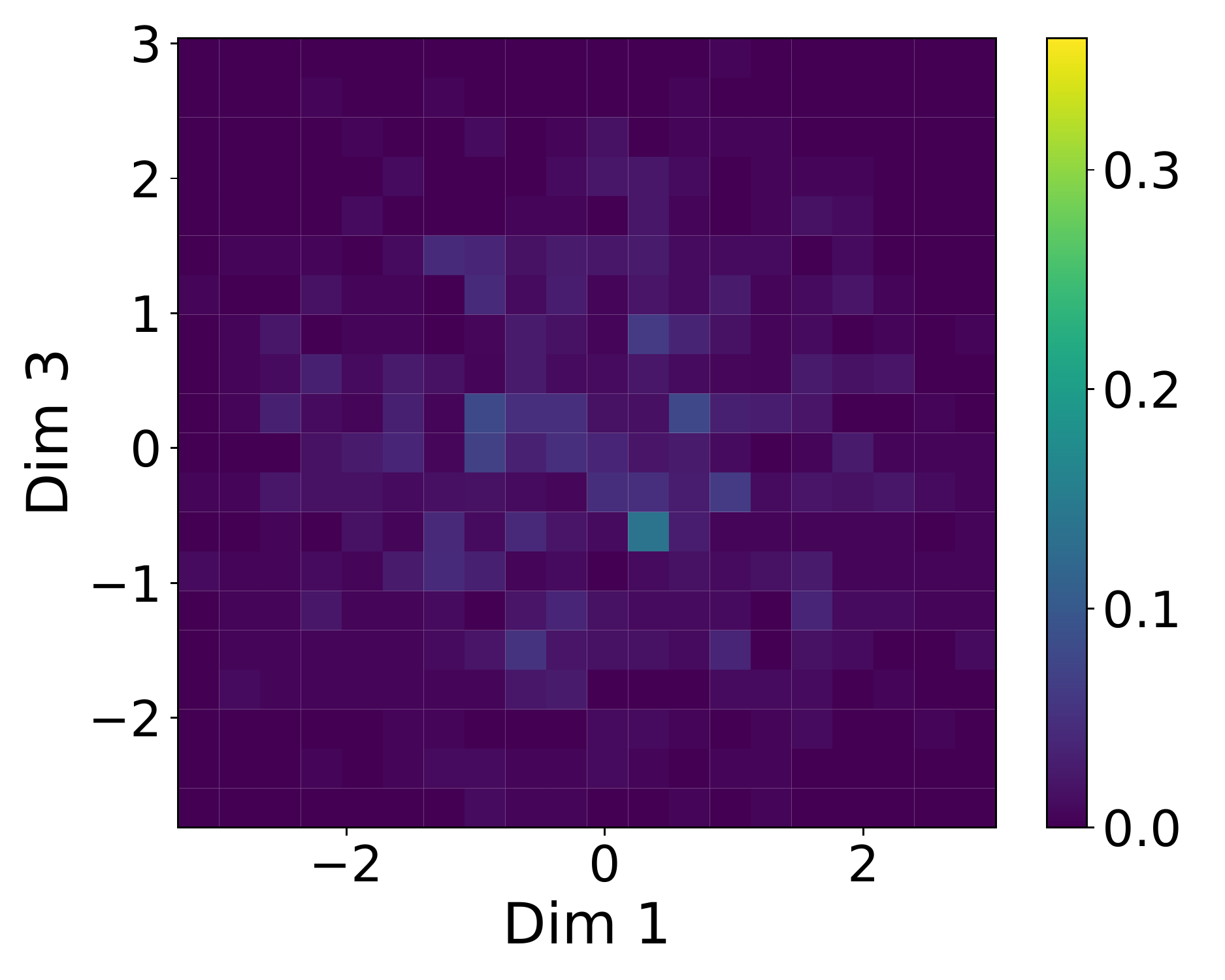}} 
\subfloat[$t=0.3$]{\includegraphics[width=0.15\textwidth]{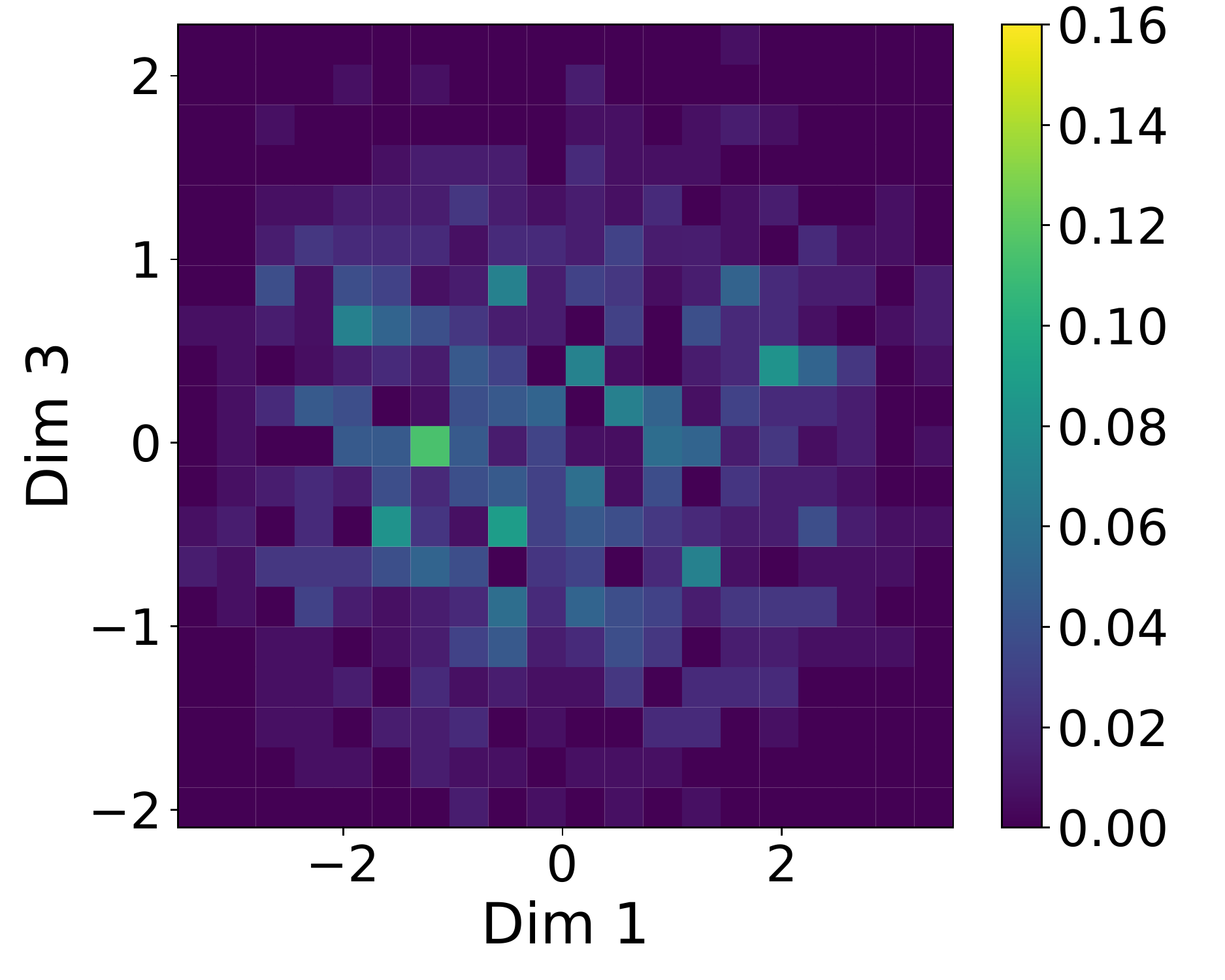}}
\subfloat[$t=0.5$]{\includegraphics[width=0.15\textwidth]{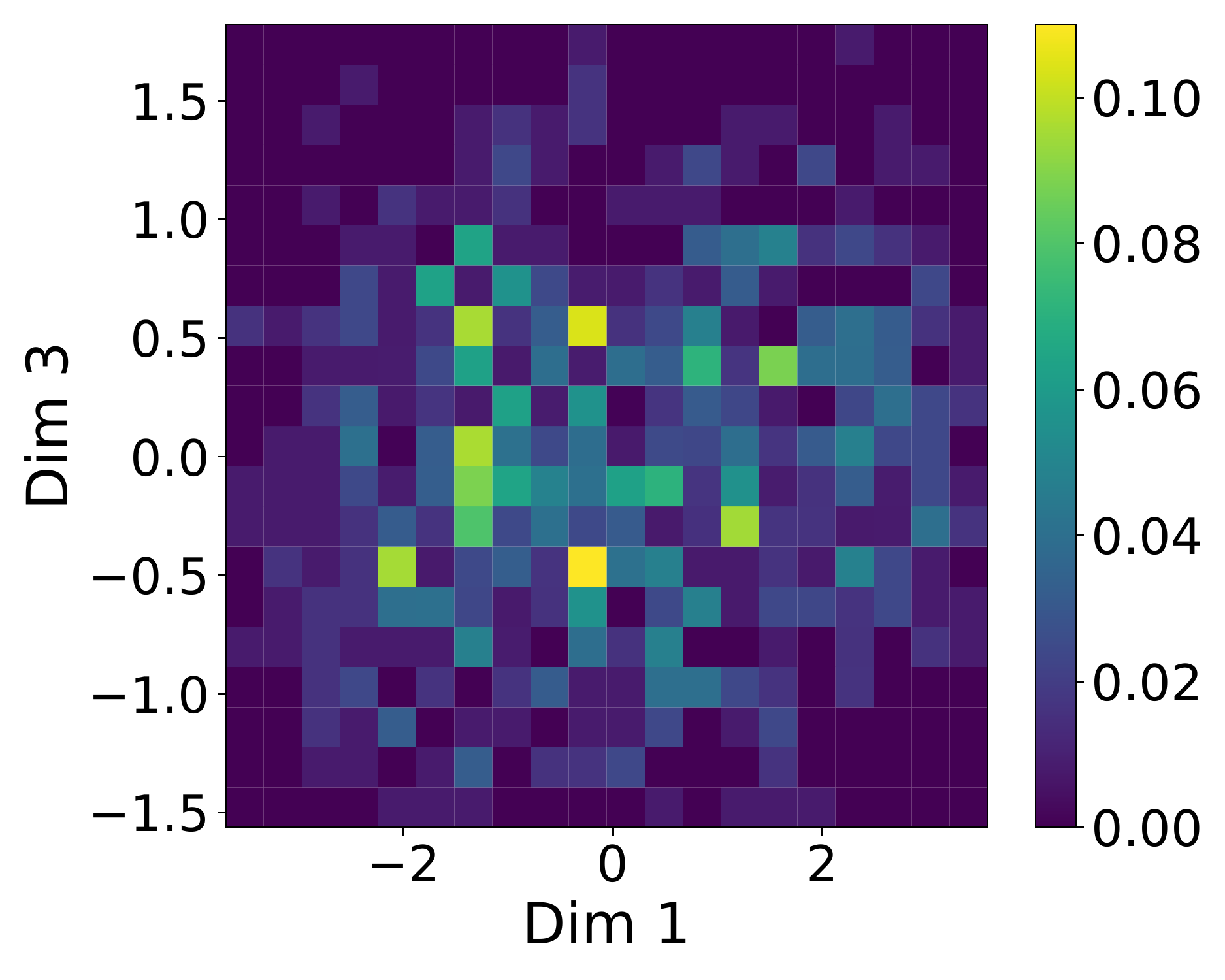}} 
\subfloat[$t=0.7$]{\includegraphics[width=0.15\textwidth]{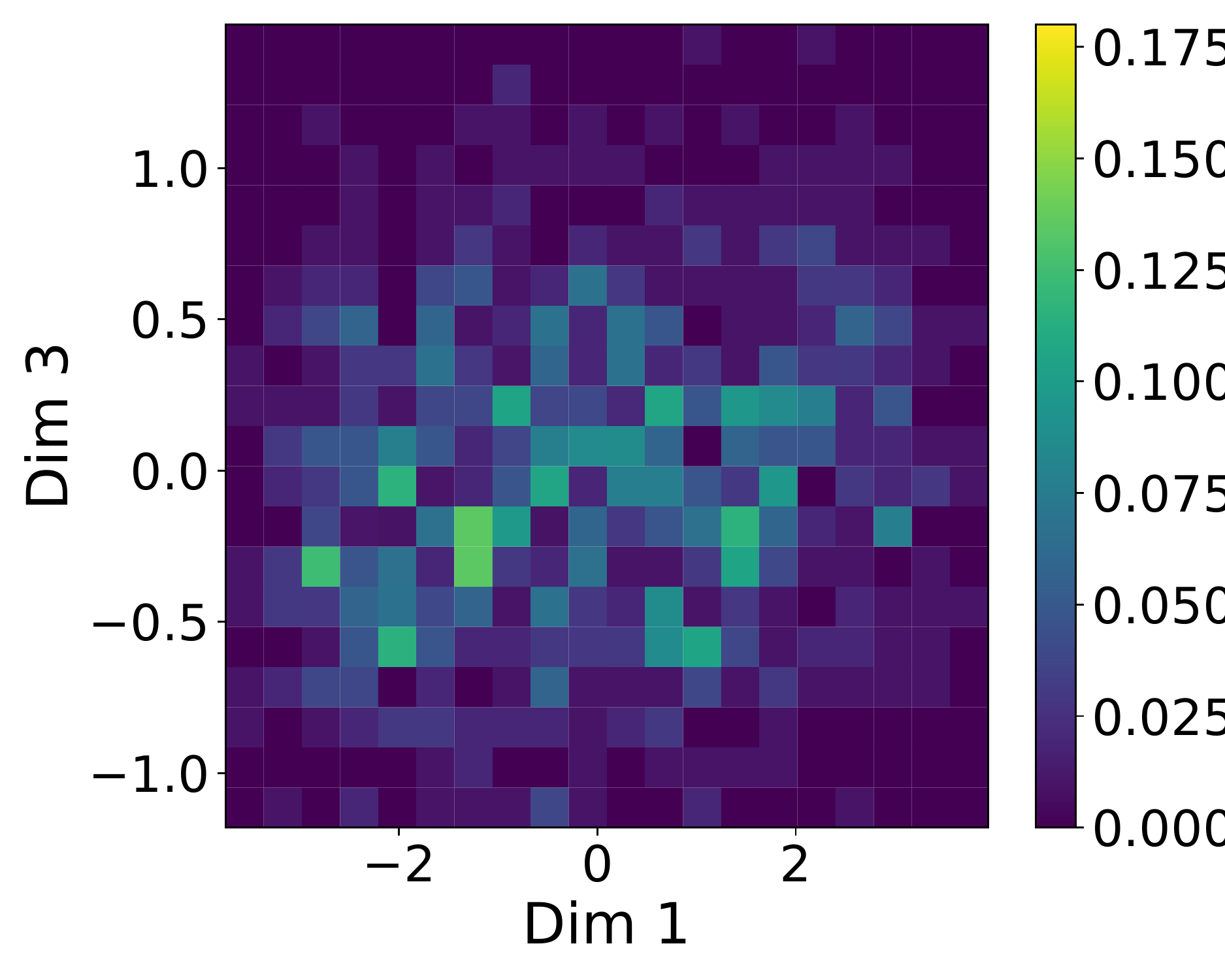}}
\subfloat[$t=0.9$]{\includegraphics[width=0.15\textwidth]{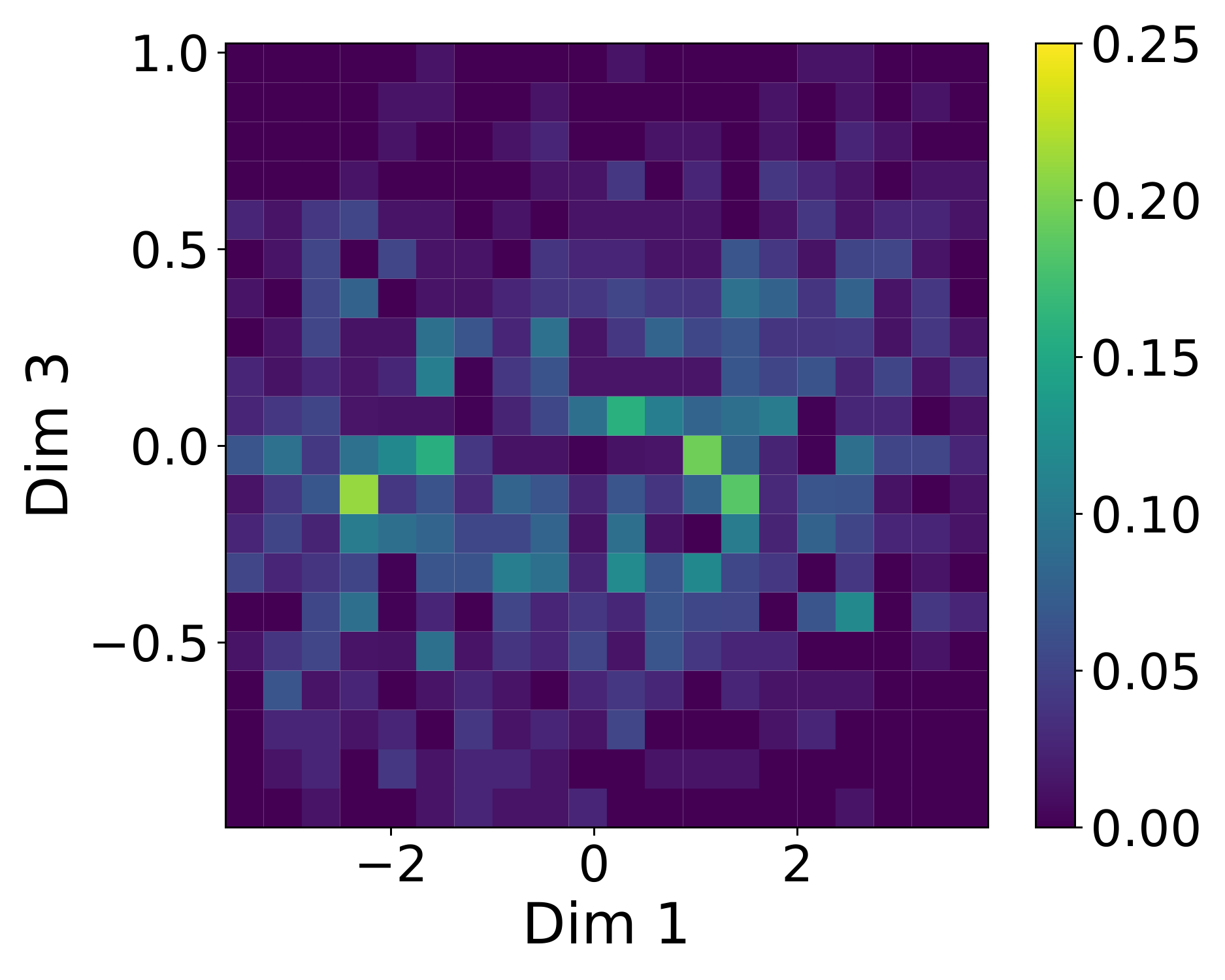}} 
\subfloat[$t=1.0$]{\includegraphics[width=0.15\textwidth]{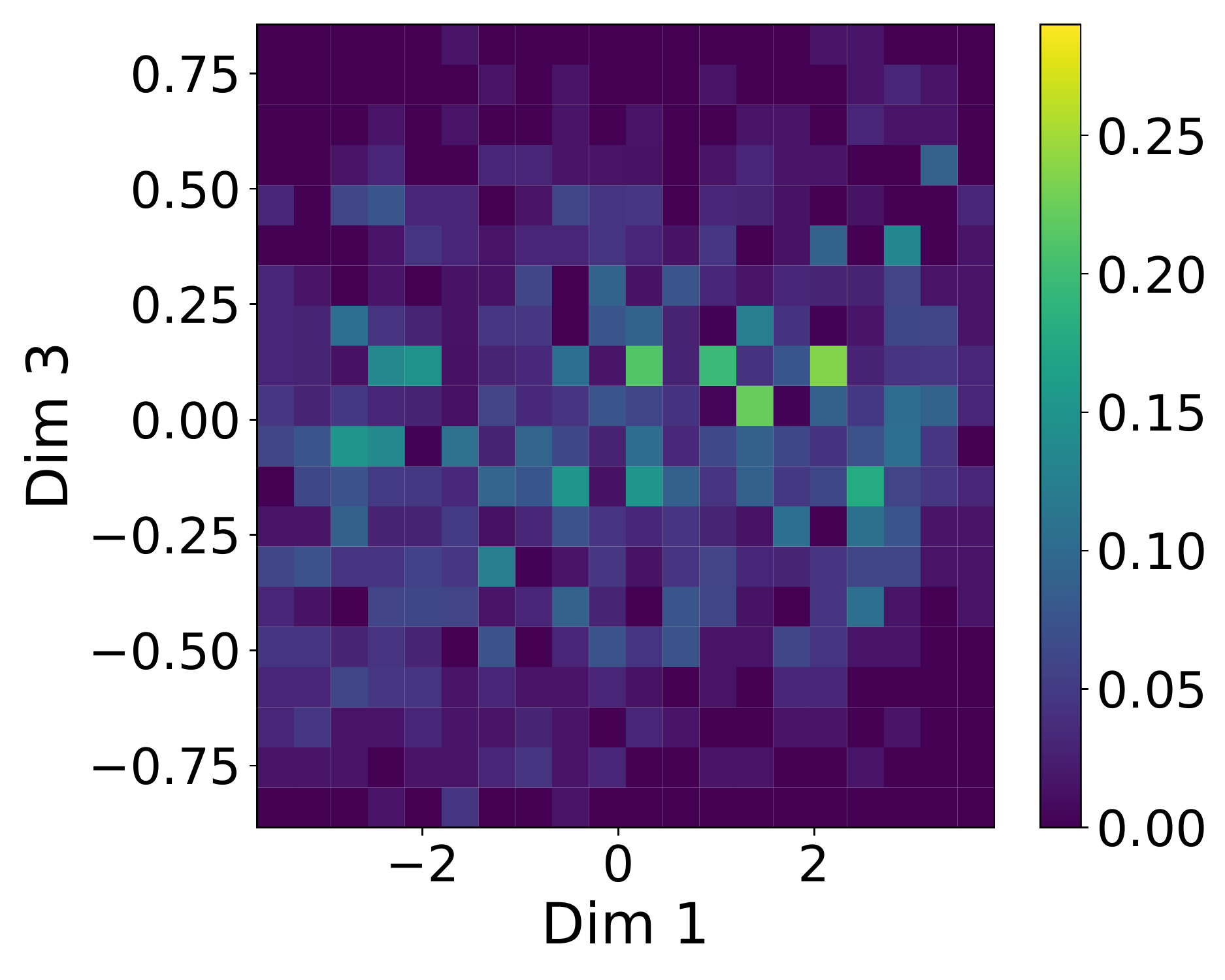}}
\vspace{-.1in}
\caption{Stochastic FitzHugh-Nagumo Model (Example 2). Figures (a)-(f) are generated by Neural SDEs, and Figures (g)-(l) are generated by DC-GANs. They show their joint marginal densities differences between estimated time series and real-time series on channels $1$ (Dim 1) and $3$ (Dim 3) at $t\in\{0.1, 0.3, 0.5, 0.7, 0.9, 1.0\}$. Darker color means a smaller difference, and thus a better fitting. One can observe that DC-GANs produce less difference in joint marginal densities at multiple time stamps.}
\label{fig:fitzhugh}
\end{figure*}



\begin{table*}[htbp]
\caption{Stochastic FitzHugh-Nagumo Model (Example 2). The scores are computed for SigWGAN, CTFP, Neural SDEs, TTS-GAN and DC-GANs under different metrics. Note that a smaller value means a better approximation. Parenthesized numbers are standard deviations. } \label{tab:fitzhugh}
\vskip 0.15in
\begin{center}
\begin{small}
\begin{sc}
\begin{tabular}{clllll}
\toprule
Method  & Discrimative & Predictive & MMD & Independence & Time (min) \\
\midrule
SigWGAN & 0.126 (0.04)  & 0.44 (0.001)  & 0.737 (0.01)& 0.0083(0.0024) & 9.63\\
CTFP & 0.275 (0.05)  & 0.501 (0.004) & 1.095 (0.02) & 0.0088(0.0023) & 6.88\\
Neural SDEs             &  0.20 (0.003) & 0.44 (0.000) & 0.97 (0.02) & 0.0085 (0.0023) & 8.25 \\
TTS-GAN & 0.258(0.02) & 0.45(0.001) & 0.96(0.04) & 0.0082(0.002) & 16.3 \\
DC-GANs             & \textbf{0.01 (0.009)} & \textbf{0.439 (0.000)} & \textbf{0.47 (0.02)} & 0.0085 (0.0027) & 8.13 \\
\bottomrule
\end{tabular}
\end{sc}
\end{small}
\end{center}
\vskip -0.1in
\end{table*}




\subsection{Example 3: Stock Price Time Series (Real Data)}
\begin{table*}[!t h]
\caption{Stocks Price Time Series (Example 3). The scores are computed for SigWGAN, CTFP, TimeGAN, Neural SDEs, TTS-GAN and DC-GANs under different metrics. Note that a smaller value means a better approximation. Parenthesized numbers are standard deviations. } \label{tab:stock}
\vskip 0.15in
\begin{center}
\begin{small}
\begin{sc}
\begin{tabular}{clllll}
\toprule
Model  & Discriminative &  Predictive & MMD & Independence  & Time (min) \\
\midrule
SigWGAN & 0.183 (0.03)  & 0.060 (0.004)  & 0.121 (0.011) & 0.012(0.004)& 4.13 \\
CTFP & 0.256 (0.05)  & 0.138 (0.006) & 0.187 (0.009) & 0.013(0.005) & 6.40 \\
TimeGAN         & 0.102 (0.021) & 0.038 (0.001)   & 0.0220 (0.007) & 0.011 (0.005) & $>$660\\
Neural SDEs             & 0.085 (0.028) & 0.048 (0.001) & 0.0193 (0.008) & 0.011 (0.006) & 9.93\\
TTS-GAN & 0.093(0.022) & 0.041(0.001) & 0.023(0.007) & 0.010(0.004) & 19.2 \\
DC-GANs             & \textbf{0.045 (0.015)} & \textbf{0.036 (0.000)} & \textbf{0.0133 (0.005)} & 0.013 (0.006) & 9.53\\
\bottomrule
\end{tabular}
\end{sc}
\end{small}
\end{center}
\vskip -0.1in
\end{table*}
The third example is Google stock prices from 2004 to 2019, extracted from Yahoo Finance. Sequences of stock prices are known as continuous time series data with unknown distributions, and can even be non-Markovian. Our data have six channels, volume and high, low, opening, closing, and adjusted closing prices. Among all, the first five channels are multimodal.
The combined discriminator \eqref{eq:combined_loss} (Neural CDE and $\mathrm{Sig\mbox{-}W}_1$) is used in GAN for this experiment, and we list the comparison results in Table \ref{tab:stock}. One can see that DC-GANs outperform SigWGAN, CTFP, TimeGAN, Neural SDEs and TTS-GAN under all three metrics. 


\subsection{Example 4: Energy Consumption Data (Real Data)}
We download the Energy Consumption data from Ireland's open data portal, and choose four electric and gas consumption time series from 02/2011–02/2013, where channels 1,3, and 4 exhibit multimodal features. We list the comparison results in Table \ref{tab:energy}, which shows consistent advantages of DC-GANs compared with other methods under different metrics as in previous examples. Notice that DC-GANs can be used with both Neural CDEs (NCDE) and Signature Wasserstein (SigW) discriminators, and in this example, DC-GANs with SigW as the discriminator present better performance and have a faster running time.




\begin{table*}[!h t]
\caption{Energy Consumption Data from Ireland's open data portal (Example 4). The scores are computed for SigWGAN, CTFP, Neural SDEs, and DC-GANs under different metrics. Note that a smaller value means a better approximation. Parenthesized numbers are standard deviations.}
\label{tab:energy}
\vskip 0.15in
\begin{center}
\begin{small}
\begin{sc}
\begin{tabular}{lccccr}
\toprule
Method  & Discriminative & Predictive & MMD & Independence & Time(min) \\
\midrule

SigWGAN & 0.368 (0.09)  & 0.159 (0.002)  & 0.135 (0.006) & 0.022(0.007) & 9.47\\
  CTFP & 0.487 (0.01)  & 0.185 (0.001) & 0.558 (0.006) & 0.021(0.008) & 8.52 \\
 Neural SDEs        &   0.413 (0.06)  & 0.172 (0.004)  & 0.126 (0.004) & 0.022(0.006) & 9.73 \\
 TTS-GAN & 0.394(0.04) & 0.167(0.003) & 0.183(0.008) & 0.022(0.006) & 17.4\\
 DC-GANs (w/ NCDE)   &   {0.322 (0.12)}     &  {0.155 (0.006)} & {0.077 (0.003)} & 0.029(0.007) & 23.44  \\
 DC-GANs (w/ SigW)   &   {\bf 0.310 (0.09)}     &  {\bf 0.151 (0.008)} & {\bf 0.075 (0.003)} & 0.033(0.008) & 9.38  \\
\bottomrule
\end{tabular}
\end{sc}
\end{small}
\end{center}
\vskip -0.1in
\end{table*}

\subsection{Dependence Elimination \& Ablation Study.} To demonstrate the effectiveness of removing dependence in the {\it Decorrelating and Branching Phase} (Section~\ref{sec:generator}), we compare the independence score \eqref{eq:independence_score} with choices of $q=2$ (basic model) and $q=10$ (DC-GANs), and present the results in Table \ref{tab:independence_compare}. A smaller score indicates better independence. The large differences observed in all four experiments suggest that the proposed scheme significantly reduces correlation in the directed chain generator. We anticipate that this approach can be employed in other directed chain-related methods to effectively mitigate strong dependence on generated data.





\begin{table*}[!h t]
    \caption{The comparison of independence scores between a basic model (with $q=2$) and DC-GANs (with $q=10$). A smaller score indicates better independence.}
    \label{tab:independence_compare}
    \vskip 0.15in
    \centering
    \begin{sc}
    \begin{tabular}{cllll} \hline
       Experiments  & Opinion (Exp.1) & Fitz-Nag (Exp.2) & Stock (Exp.3) & Energy(Exp.4)\\ \hline 
       $q=2$ (Basic) &0.078(0.013) & 0.052(0.009)& 0.130(0.024) & 0.119(0.017) \\
       $q=10$ (DC-GANs) & 0.009(0.004) & 0.0085(0.0027) & 0.013(0.006) & 0.033(0.008) \\ \hline
    \end{tabular}
    \end{sc}
\end{table*}


 We also conduct an ablation study on the discriminator by using an ordinary LSTM as the discriminator and implementing Wasserstein-GAN for comparison. The results, summarized in Table~\ref{tab:wgan}, indicate that DC-GANs (shown in Tables~\ref{tab:opinion}-\ref{tab:energy}) outperform Wasserstein-GAN with an LSTM discriminator in both accuracy and speed.


\begin{table*}[!h t]
\caption{The scores computed when the discriminator is replaced by an ordinary WGAN. Parenthesized numbers are standard deviations. Compared to Tables~\ref{tab:opinion}-\ref{tab:energy}, DC-GANs outperform Wasserstein-GAN with an LSTM discriminator in both accuracy and speed.}\label{tab:wgan}
\vskip 0.15in
    \centering
    \begin{sc}
    \begin{tabular}{clllll} \hline
       Experiments  & Discriminative & Predictive & MMD & Independence& Time (min) \\ \hline 
       Stoch. Opinion  & 0.041(0.016) & - & 0.123(0.004) & 0.008(0.003) & 19.5\\ 
       FitzHugh-Nagumo & 0.09(0.01) & 0.441(0.000) & 0.66 (0.07) &  0.0083(0.0022) & 26.2 \\ 
       Google Stock & 
0.082 (0.024) & 0.059 (0.002) & 0.0176 (0.006) & 0.010 (0.005)  & 20.85 \\ 
       Energy Consumption & 0.343 (0.13) & 0.161(0.004) & 0.093 (0.004) & 0.024(0.004) & 21.17 \\ \hline
    \end{tabular}
    \end{sc}
\end{table*}

\section{Conclusion}

We propose a novel time series generator, DC-GANs, motivated by the study of \citet{detering2020directed, ichiba2022smoothness} on directed chain SDEs (DC-SDEs). Compared to more complicated graph systems, we find from numerical examples that the directed chain systems exhibit promising ability in fitting time series of multimodal probability distributions. We prove in theory that DC-GANs have the same flexibility as the Neural SDEs in capturing marginal distributions, and DC-GANs naturally embrace the non-Markovian property in the topological structure, if needed. We also prove that the correlation of the generated path decays exponentially fast as the graph distance of the generated path from the original data becomes large under some mild assumptions, and hence, the lack-of-independence problem can be overcome by walking along the directed chain. We present four numerical examples, two synthetic datasets generated by the SDEs, and two real-world data of stock price and energy consumption, and show that DC-GANs have a better performance than SigWGAN, CTFP, Neural SDEs, TimeGAN and TTS-GAN, with the comparable independence property. {We remark that the DC-GANs algorithm can also work with irregular data (i.e., the sample paths may have data sampled on different time grids), which may happen in healthcare applications. }

{\bf Potential Societal Impact.} The proposed DC-GANs in this paper offer a fast and flexible generative adversarial network method for machine learning research areas, such as biology, economics, environmental science, finance, medicine, and more, where path-dependent analysis is critical. They have the potential to contribute to research areas where sequential data is scarce or missing. Furthermore, the strong predictive power of DC-GANs demonstrated in the energy consumption example (Example 4) can aid in the development of energy management and help reduce energy waste. As DC-GANs are primarily used to generate time-series data or sequential data, rather than fake faces, the usual negative social impact associated with creating fake social media accounts to spam would not be a concern in this context.





\section*{Acknowledgements}

R.H. was partially supported by the NSF grant DMS-1953035, and the Early Career Faculty Acceleration funding and the Regents’ Junior Faculty Fellowship at the University of California, Santa Barbara. T.I. was partially supported by NSF grant DMS-2008427. The authors are grateful to the reviewers for their valuable and constructive comments. 

\newpage
\bibliography{citations}
\bibliographystyle{icml2023}

\newpage
\appendix
\onecolumn
\clearpage
\appendix
\thispagestyle{empty}
\onecolumn
\section{Preliminaries on Directed Chain SDEs}\label{app:prelim}
In this appendix, we give an intuitive explanation of how the limit of an $n$-coupled SDE system leads to the DC-SDE. As before, we use $X_s, X_t$ to denote the state of $X$ at time $s$ and $t$, respectively. With no subscript, e.g., by $X$, we mean the whole path from $t=0$ to $T$. With a little abuse of notations, we use $X_i, X_{i+1}$ to denote the $i$-th or $(i+1)$-th node, and  with two subscripts, e.g., $X_{i,t}$, it represents the state value of the $i$-th node at time $t$.

Let us start with a system of $n$-coupled SDEs, which approximates a generic directed chain SDE when $n$ goes infinity. An illustration of their chain-like coupling is given in Figure \ref{fig:chain}(a). Each node $X_i$ satisfies an SDE, which also depends on the node pointing to it. More specifically, for $i \leq n-1$, $X_{i+1}$ depends on $X_{i}$, and we say the $i$-th node is the neighborhood of the $(i+1)$-th node; the $n$-th node affects the first one, yielding a circular structure. The dependence is determined in a homogeneous manner over the whole system of particles. Such a circular chain structure forces every node to be identically distributed. When $n$ goes to infinity, the circle chain is equivalent to a non-circular chain with the {\it distribution constraint} (see Definition \ref{def:dc_sde}), and we get the abstract DC-SDE. We refer interested readers to  \citet{detering2020directed, ichiba2022smoothness} for more details.

\begin{figure}[ht]
\centering
\subfloat[Circular structure in the $n$-coupled SDEs]{\quad \qquad \includegraphics[width=0.21\textwidth]{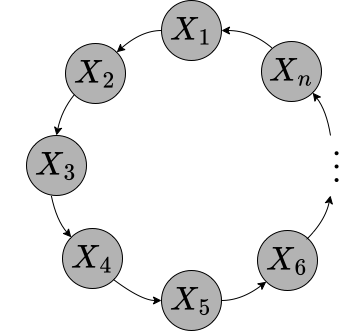}\qquad \quad} \\
\vspace{.2in}
\subfloat[None-Circular structure in the DC-SDEs ]{\includegraphics[width=0.35\textwidth]{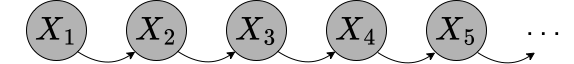}}
\vspace{.1in}
\caption{Illustrative Directed Chain Structure. Each node $X_i$ satisfies an SDE, which also depends on the node pointing to it.}
\label{fig:chain}

\end{figure}

\section{Additional Theorems and Proofs}\label{app:proof}

\subsection{Property of Signatures}
\label{app:sig_decay}
We provide some related properties and theorems of signatures.
Firstly, we justify the validity of truncating signatures in Section~\ref{sec:SignatureDef}.  The signature of a path is an infinite series of iterated integrals, which can be used to represent the path. Practically, one can only deal with finite sequences, thus requiring truncating signatures up to a finite order (called the signature depth). This truncation order is, in general, determined by its factorial decay property, which is indicated in the following extension theorem. For the formal definition of control $\omega$ and $\beta$, we refer interested readers to the book by \citet{lyons2002system}. 

\begin{thm}[Extension Theorem,  {\citet[Theorem 3.7]{lyons2002system}}] 
\label{extension_thm} 
Let $p\ge 1$ be a real number and $n\ge 1$ an integer with $n\ge\lfloor p \rfloor$. Denote $\bX: \Delta_T \to T^n(\RR^d)$ as a multiplicative functional with finite $p$-variation controlled by a control $\omega$. 
Then there exists a unique extension of $\bX$ to a multiplicative functional $\Delta_T\to T((\RR^d))$ which possesses finite $p$-variation.

More precisely, for every $m\ge \lfloor p\rfloor + 1$, there exists a unique continuous function $\bX^m:\Delta_T\to (\RR^d)^{\bigotimes m}$, such that
$$(s,t) \to \bX_{s,t}=\left( 1, \bX^1_{s,t}, \dots, \bX^{\lfloor p\rfloor}_{s,t}, \dots, \bX^m_{s,t}, \dots \right) \in T((\RR^d))$$
is a multiplicative functional with finite $p$-variation controlled by $\omega$, i.e.,
\begin{equation}
\label{extension_eq}
	\|\bX^i_{s,t}\| \le \frac{\omega(s,t)^{\frac{i}{p}}}{\beta (\frac{i}{p})!} \,\,\,\, \forall i\ge 1, \,\,\,\, \forall (s,t)\in \Delta_T.
\end{equation}
\end{thm}

The extension theorem states that, for any multiplicative functional with finite $p$-variation, we can extend it to an infinite sequence. In particular, the signature is one of  such objects for some special multiplicative functionals (what we call geometric rough paths).
In most real-world applications, time series data are interpolated linearly and hence fall into the case of $p=1$, i.e., paths with bounded variation. In certain financial applications or the first two examples in this paper, we have semi-martingales that fall into the cases of $p\in(2,3)$, that is, the geometric rough paths with finite $p$-variation. Their signatures are all well-defined. 
The factorial decay property is implied by equation \eqref{extension_eq}.


{
As a feature map of sequential data, the signature has a
universality detailed in the following theorem.

\begin{thm}[Universality]
    Let $p\ge 1$ and $f: \mathcal{V}^p([0,T], \R^d)\to \R$ be a continuous function in paths. For any compact set $K\subset \mathcal{V}^p([0,T], \R^d)$, if $S(x)$ is a geometric rough path for any $x\in K$, then for any $\eps >0$ there exist $M>0$ and a linear functional $l \in T((\RR^d))^\ast$ such that
    \begin{equation}
        \sup_{x\in K} |f(x) - \langle l, S(x) \rangle| <\epsilon.
    \end{equation}
    \label{thm:sig_universality}
\end{thm}

Given that signature is well-defined and with finite expectation, we call $\EE[S(X)]$ the expected signature of $X$. Intuitively, the expected signature serves the moment-generating function, which can characterize the law induced by a stochastic process under some regularity conditions. More precisely, an immediate consequence of Proposition 6.1 in \citet{chevyrev2018signature} on the uniqueness of the expected signature is summarized in the below theorem:

\begin{thm}
    Let $X, Y$ be two random variables of geometric rough paths such that $\EE[S(X)]] = \EE[S(Y)]$ and $\EE[S(X)]$ has an infinite radius of convergence, then $X, Y$ have the same distribution. 
\end{thm}

}

\subsection{Proof of Theorem \ref{thm:flexibility}}\label{app:flexibility}
We first restate Theorem \ref{thm:flexibility} formally. Without loss of generality, we treat the time-homogeneous case, i.e.,  $\mu$ and $\sigma$ are independent of $t$. 
Our proof relies on constructing the forward equations characterizing marginal distributions of both SDEs and directed chain SDEs, thus can be easily generalized to  time-dependence cases. The forward equation associated with directed chain SDEs has been constructed by \citet{ichiba2022smoothness} and will be used directly in our proof. 
\begin{thm}
Let $Y\in L^2(\Omega\times[0,T], \RR^N)$ be an $N$-dimensional stochastic process with the following dynamics
$$\ud Y_t = \mu(Y_t)\ud t + \sigma(Y_t) \ud B^y_t, \quad Y_0=\xi^y,$$
where $B^y$ is a standard $d$-dimensional Brownian motion, and $\mu : \RR^N\to \RR^N, \sigma : \RR^N \to \RR^{N\times d}$ are Borel measurable functions with Lipschitz and linear growth conditions. Then, there exist functions $V_0$ and $V_1$ such that the process $X$ has the same marginal distribution as $Y$ for all $t\in [0, T]$, where $X$ is described by the following directed chain SDEs with an initial position $\xi$ as an independent copy of $\xi^y$,
\begin{align*}
    &\ud X_t = V_0(X_t, \Tilde{X}_t)\ud t + V_1(X_t, \Tilde{X}_t) \ud B^y_t, \quad X_0=\xi, \\
    &\text{subject to: } \Law(X_t, 0\le t\le T) = \Law(\Tilde{X}_t, 0\le t\le T).
\end{align*}

\end{thm}

\begin{proof}
Let $g\in\mc{C}^2(\RR^N)$ be a twice continuously differentiable function. To characterize marginal distributions of the SDE solution $Y$ for all $t\in[0,T]$, we use the Kolmogorov forward equations. Define $u(t, x):= \EE[g(Y_t)|Y_0=x]$, it is the solution of the following Cauchy problem
\begin{align}
    (\partial_t - \mc{L})u(t, x) &= 0, \label{eq:pde1}\\
     u(0, x) & = g(x).\label{eq:bd1}
\end{align}
The derivation relies on It\^o's formula and can be found in stochastic calculus textbooks, e.g., in  \citet{karatzas2012brownian}. Here the \textit{infinitesimal operator} $\mc{L}$ is given by
$$\mc{L}g(x) = \mu(x) \cdot \nabla_x g(x) + \frac{1}{2}\text{Tr}(\sigma \sigma^T(x) \text{Hess}_{x} g(x)),$$
where $\text{Hess}_x(\cdot)$ denotes the Hessian matrix, and $\text{Tr}(\cdot)$ denotes the matrix trace. In \citet[Section 4.5]{ichiba2022smoothness}, a similar partial differential equation for the directed chain SDEs is derived, and we here summarize a simpler version without the mean-field interaction term. Define $v(t, x):=\EE[g(X_t)|X_0=x]$, then $v$ solves 
\begin{align}
    (\partial_t - \mc{L}^{dc})v(t, x) &= 0, \label{eq:pde2}\\
     v(0, x) &= g(x).\label{eq:bd2}
\end{align}
Let $\Tilde{\xi}$ be an independent copy of $\xi$, and the differential operator $\mc{L}^{dc}$ is given by
\begin{align}
\label{eq:operator_L}
    \mc{L}^{dc} g(x) =&\EE_{\tilde{\xi}}\bigg[V_0(x, \tilde{\xi})\cdot \nabla_{x}g(x) + \frac{1}{2} \text{Tr}(V_1 V_1^T(x, \tilde{\xi})\text{Hess}_{x}g(x))\bigg],
\end{align}
where $\EE_{\tilde{\xi}}$ is the expectation with respect to the distribution of $\tilde{\xi}$. 
As long as we can match these two operators $\mc{L}$ and $\mc{L}^{dc}$ with some non-degenerate choices of $V_0, V_1$, then \eqref{eq:pde1}-\eqref{eq:bd1} and \eqref{eq:pde2}-\eqref{eq:bd2} agree with each other and so do their solutions $u$ and $v$. To this end, it suffices to choose $V_0, V_1$ such that
\begin{align*}
    &  \EE_{\Tilde{\xi}}[V_0(x, \tilde{\xi})] = \mu(x),\\
    & \EE_{\Tilde{\xi}}[V_1 V_1^\top(x, \tilde{\xi})] = \sigma\sigma^\top(x).
\end{align*}
A toy example of non-degenerate $V_0, V_1$ can be $V_0(x, \tilde{\xi}) = \mu(x) + \varphi_1(\tilde{\xi})-\EE_{\tilde{\xi}}[\varphi_1(\tilde{\xi})]$ and $V_1(x, \tilde{\xi})$ such that $ V_1 V_1^\top(x, \tilde{\xi}) = \sigma\sigma^\top(x) + \varphi_2(\tilde{\xi})-\EE_{\tilde{\xi}}[\varphi_2(\tilde{\xi})]$
with measurable and integrable functions $\varphi_1, \varphi_2$. 
\end{proof}

\subsection{Proof of Theorem \ref{thm:dependence_decay}}\label{app:dependence-decay}
From \citet[Proposition 2.1]{ichiba2022smoothness}, we have the existence and weak uniqueness of directed chain SDEs. Denote this unique measure flow by 
$$m := \Law(X_t, 0\le t\le T) = \Law(\Tilde{X}_t, 0\le t\le T).$$
This measure can also be understood as a probability distribution on $C([0,T], \RR^N)$.
Given the Brownian motion path and the neighborhood path, we define a map $\Phi: C([0,T], \RR^N)\times C([0,T], \RR^d) \to C([0,T], \RR^N)$ such that
$$X = \Phi(\Tilde{X}; B) \in C([0,T], \RR^N)$$
and $\Phi_t$ as the projection of $\Phi$ onto any specific time stamp, i.e. $X_t \equiv \Phi_t(\Tilde{X}; B)$. Then, on a chain-like structure depicted in Figure \ref{fig:chain}{(b)} or \ref{fig:branch}, we write
$$X_q = \Phi(X_{q-1}; B^{q}) = \Phi( \Phi(X_{q-2}; B^{q-1}); B^{q}) =\Phi \circ \Phi (X_{q-2}; B^{q-1}, B^{q}).$$
Namely, $X_q$ is obtained as an output of the composite map $\Phi\circ\Phi$ from the inputs $X_{q-2}$, $B^{q-1}$ and $B^q$. 
Repeating the above equation until tracing back to the first node produces
$$X_q = \Phi\circ \cdots \circ \Phi(X_1; B^{2}, \dots, B^{q}) := \Phi^{q}(X_1; \mathbf{B}),$$
where $\mathbf{B}=(B^{2}, \dots, B^{q})$ and $B^2, \ldots, B^q$ are independent $d$-dimensional Brownian motions.
Such a chain-like structure possesses \textit{local Markov property} as pointed out in Proposition 4.6 in \citet{ichiba2022smoothness}.  Let us denote $X_{t, q}=\Phi^q_t(X_1; \mathbf{B})$. In the proof below, we impose Lipschitz and linear growth conditions on coefficients $V_0$ and  $V_1$.
\begin{assump} \label{assump:lip_linear}
For both coefficients $V_0$ and $V_1$, there exists a positive constant $C_T$ such that,
\begin{enumerate}
    \item (Lipschitz conditions) for $i=0,1$, 
    $$\lvert V_i(x_1, y_1) - V_i(x_2, y_2) \rvert \le C_T(|x_1-x_2| + |y_1-y_2|);$$
    \item (Linear growth conditions) for $i=0,1$, 
    $$V_i(x, y) \le C_T(1+|x|+|y|).$$
\end{enumerate}
    
\end{assump}

The following lemma gives the necessity of having the Brownian motion noises $B^{j}$, $j = 2, \dots, q$ in $\Phi^q_t$, in order to have dependence decay properties. 
\begin{lem}\label{lem:degenerate_case}
Suppose Assumption \ref{assump:lip_linear} holds. In the degenerate case, i.e., $V_1\equiv 0$ and $X_{t, q}=\Phi^q_t(X_1)$, if all the initial conditions $X_{0,1}=X_{0, 2}=\cdots=X_{0, q}=\xi$ are identical, then the directed chain SDE satisfy $X_1=X_2=\cdots X_q$ in the $L^2$ sense. 
\end{lem}
\begin{proof}
We first write our directed chain dynamics in the integral form,
\begin{equation}
\label{eq:degenerate_sde}
X_{t, q} = \xi + \int_{0}^t V_0(X_{s, q}, X_{s, q-1})\ud s.
\end{equation}
Note that the current directed chain system with degenerate $V_1$ also has unique solutions. By the Lipschitz property on $V_0$, we compute
\begin{align*}
    \EE[\sup_{0\le s\le t}| X_{s,q}-X_{s,q-1} |^2] &\le \EE\bigg[\sup_{0\le s\le t} 2C_T \int_{0}^s (|X_{v,q}-X_{v,q-1}|^2 + |X_{v,q-1}-X_{v,q-2}|^2)\ud v\bigg] \\
    & \le C\cdot \EE\bigg[ \int_{0}^t \sup_{0\le v\le s} \bigg( |X_{v,q}-X_{v,q-1}|^2 + |X_{v,q-1}-X_{v,q-2}|^2 \bigg)\ud v \bigg] \\
    & \le C\cdot  \int_{0}^t \EE\big[\sup_{0\le v\le s} |X_{v,q}-X_{v,q-1}|^2 \big]\ud v  + C\cdot  \int_{0}^t\EE\big[ \sup_{0\le v\le s} |X_{v,q-1}-X_{v,q-2}|^2\big] \ud v  \\
    &\le C \cdot e^{CT} \int_{0}^t\EE\big[ \sup_{0\le v\le s} |X_{v,q-1}-X_{v,q-2}|^2\big] \ud v,
\end{align*}
where the third inequality comes from Fubini's theorem and Proposition 2.2 in \citet{ichiba2022smoothness}, and the last inequality follows from Gronwall's inequality. Iterating back to the beginning of the chain, we deduce
$$\EE[\sup_{0\le s\le T}| X_{s,q}-X_{s,q-1} |^2] \le \frac{TC^{q-1}e^{(q-1)CT}}{(q-1)!} \EE[\sup_{0\le s\le T}| X_{s,2}-X_{s,1} |^2].$$
According to the invariance of (joint) distribution (see \citet{detering2020directed, ichiba2022smoothness}), we get
$$\EE[\sup_{0\le s\le T}| X_{s,2}-X_{s,1} |^2] \le \frac{TC^{q-1}e^{(q-1)CT}}{(q-1)!} \EE[\sup_{0\le s\le T}| X_{s,2}-X_{s,1} |^2].$$
The constant $q$ can be arbitrarily large and hence the above inequality forms a contraction, which implies $$\EE[\sup_{0\le s\le T}| X_{s,2}-X_{s,1} |^2] = 0.$$
We then conclude $X_1=X_2=\cdots=X_q$ in the $L^2$ sense.
\end{proof}

Although the assumption of identical initials in Lemma \ref{lem:degenerate_case} is different from the general setting of directed chain SDEs, where initials should be i.i.d, it is consistent in the case that initials are deterministic. Therefore, the existence of non-degenerate $V_1$ becomes crucial, and we give the following necessary assumptions for factorial dependence decay property.

\begin{defn}[$\mc{C}_{b, \Lip}^{k,k}$]\label{def:ckk} We have the following definition for $\mc{C}_{b, \Lip}^{k,k}$:
\begin{itemize}
    \item[(a)] We use $\partial_x, \partial_y$ to denote the derivative with respect to the first and second Euclidean variables in $V_0, V_1$.
    \item[(b)] Let $V:\RR^N\times\RR^N \to \RR^{N}$ with components $V^1, \dots, V^N: \RR^N\times\RR^N\to\RR$. We say $V\in \mc{C}_{b, \Lip}^{1,1}(\RR^N\times\RR^N; \RR^N)$ if the following is true: for each $i=1,\dots,N$, $\partial_xV^i, \partial_yV^i$ exist. Moreover, assume the boundedness of the derivatives for all $(x,y)\in\RR^N\times\RR^N$,
    $$|\partial_xV^i(x, y)| + |\partial_yV(x,y)| \le C.$$
    In addition, suppose that $\partial_xV^i, \partial_yV^i$ are all Lipschitz in the sense that for all $(x,y)\in\RR^N\times\RR^N$,
    $$\begin{array}{c}
         |\partial_xV^i(x, y) - \partial_xV^i(x^\prime, y^\prime)| \le C(|x-x^\prime|+|y-y^\prime|),  \\
         |\partial_yV^i(x, y) - \partial_yV^i(x^\prime, y^\prime)| \le C(|x-x^\prime|+|y-y^\prime|),
    \end{array}$$
    and $V^i, \partial_xV^i, \partial_yV^i$ all have linear growth property,
    $$|V^i(x,y)|+ |\partial_xV^i(x, y)| + |\partial_yV^i(x, y)| < C_T(1+|x|+|y|),$$
    where $C_T$ is a constant depending only on $T$.
    \item[(c)] We write $V\in \mc{C}_{b, \Lip}^{k,k}(\RR^N\times\RR^N; \RR^N)$, if the following holds: for each $1,\dots,N$, and all multi-indices $\alpha, \beta$ on $\{1,\dots,N\}$ satisfying $|\alpha| + |\beta|\le k$, the derivative $\partial_x^{\alpha}\partial_y^{\beta}$ exists and is bounded, Lipschitz continuous, and satisfies linear growth condition.
    \item[(d)] We say $V_0\in \mc{C}_{b, \Lip}^{k,k}(\RR^N\times\RR^N)$ for short if $V_0:\RR^N\times\RR^N\to\RR^N$ satisfies (c). Let $V_1: \RR^N\times\RR^N \to \RR^{N\times d}$ with components $V_1^{1}, \dots, V_1^d:\RR^N\times\RR^N\to \RR^N$. We say $V_1\in \mc{C}_{b, \Lip}^{k,k}(\RR^N\times\RR^N)$ for short if $V_1^j \in \mc{C}_{b, \Lip}^{k,k}(\RR^N\times\RR^N)$ for every $ j=1,\dots,d$.
\end{itemize}
\end{defn}

\begin{assump}\label{assump:smoothness_density}
We emphasize two assumptions used for the existence and smoothness of the marginal densities of directed chain SDEs:
\begin{enumerate} 
    \item (Uniform ellipticity on $V_1$) Assume that there exists $\epsilon>0$ such that for all $\eta, x, \Tilde{x} \in \RR^N$,
    $$\eta^{\top} V_1(x, \Tilde{x})V_1(x, \Tilde{x})^{\top} \eta \ge \epsilon |\eta|^2.$$
    
    \item (Smoothness on $V_0, V_1$) Assume that $V_0, V_1 \in \mc{C}_{b, \Lip}^{k,k}(\RR^N, \RR^N)$ with $k\ge N+2$, where $V_0, V_1 \in \mc{C}_{b, \Lip}^{k,k}(\RR^N, \RR^N)$ is defined in Definition~\ref{def:ckk}.
\end{enumerate}
\end{assump}
Under Assumption \ref{assump:smoothness_density}, one can prove the existence of the density function of directed chain SDEs \citep[Theorem 4.3]{ichiba2022smoothness}. 

\begin{thm}\label{thm:dependence_decay_formal}
Suppose Assumption \ref{assump:smoothness_density} is satisfied. For every Lipschitz function $\varphi:\RR^N\to \RR$ with Lipschitz constant $K$, there exists a constant $c>0$ such that the difference between the conditional expectation of $\varphi(X_{t, q})$, given $X_{1}$ and the unconditional expectation $\varphi(X_{t, q})$ for all $t\in[0,T]$ is bounded, i.e.,
\begin{align}
    \EE\big[ \sup_{0\le t\le T}\big| \EE[ \varphi(X_{t,q}) |X_1] - \EE[ \varphi(X_{t,q})] \big|^2 \big] \le \frac{c^{q-1}}{(q-1)!}. \label{eq:factorial_decay}
\end{align}
\end{thm}

We shall first provide some interpretations for Theorem \ref{thm:dependence_decay_formal} before giving the proof. For random variables in space $C([0,T], \RR^N)$, there is no unique choice on how to measure their correlation or covariance. Here, we measure the difference between conditional expectation and unconditional expectation over a family of testing functions $\varphi$. Thus, we use the left-hand side in inequality \eqref{eq:factorial_decay} to measure the dependence between $X_{q}$ and $X_1$. 

\begin{proof}
Note that Assumption \ref{assump:smoothness_density} is a stronger version of Assumption \ref{assump:lip_linear}, and it not only ensures the existence and weak uniqueness of the solution, but also guarantees the existence of a smooth density which excludes the case of deterministic $X_{t,q}$. If $X_{q}$ and $X_1$ are independent, the left-hand side is zero for every Lipschitz function $\varphi$. The vice versa is also correct because of the exclusion of the deterministic case.

Let us start from the left-hand side in \eqref{eq:factorial_decay}, the difference between the conditional expectation of $\varphi$ and unconditional expectation can be bounded by 
\begin{align*}
    &\EE\big[ \sup_{0\le t\le T}\big| \EE[ \varphi(X_{t,q}) |X_1] - \EE[ \varphi(X_{t,q})] \big|^2 \big] \\
    & = \int_{C([0,T], \RR^N)}\sup_{0\le t\le T} \bigg|  \int_{C([0,T], \RR^N)} \bigg( \EE_{\bf{B}}\big[\varphi(\Phi_t^q(\omega; \mathbf{B}))\big] - \EE_{\bf{B}}\big[\varphi(\Phi_t^q(\tilde{\omega}; \mathbf{B}))\big] \bigg)  m(\ud \tilde{\omega}) \bigg|^2 m(\ud \omega) \\
    & \le \int_{C([0,T], \RR^N)^2} \sup_{0\le t\le T} \EE_{\mathbf{B}}\big[ \big| \varphi(\Phi_t^q(\omega; \mathbf{B})) - \varphi(\Phi_t^q(\tilde{\omega}; \mathbf{B})) \big|^2 \big]m(\ud \tilde{\omega}) m(\ud \omega) \\
    & \le K^2 \int_{C([0,T], \RR^N)^2} \EE_{\mathbf{B}}\big[\sup_{0\le t\le T}\big|\Phi_t^q(\omega; \mathbf{B}) -  \Phi_t^q(\tilde{\omega}; \mathbf{B}) \big|^2 \big]m(\ud \tilde{\omega}) m(\ud \omega) \\
    &\le \frac{c^{q-1}}{(q-1)!},
\end{align*}
for some positive constant $c$, where the proof of the last inequality is verbatim to the procedures in Lemma \ref{lem:degenerate_case}.
\end{proof}

\begin{rem}\label{rem:degenerate_eg}
We shall emphasize that the assumption $X_{0, 1}=X_{0, 2}=\cdots=X_{0, q}=\xi$ is not allowed under directed chain framework except for $\xi\equiv x\in\RR^N$ (the deterministic initial condition). This is quite common in practice, for instance, the investment returns usually start from 1. Given results from Lemma \ref{lem:degenerate_case} and equation \eqref{eq:factorial_decay}, we are able to conclude that
$$\EE[\sup_{0\le t\le T} |\varphi(X_1) - \EE[\varphi(X_1)]|^2] \le \frac{c^{q-1}}{(q-1)!}.$$
Here $q$ can be arbitrarily large, hence we conclude that $\EE[\sup_{0\le t\le T} |\varphi(X_1) - \EE[\varphi(X_1)]|^2] = 0$. The only possible solution for a directed-chain system is the deterministic case where we have deterministic initial conditions and degenerate $V_1$ (or we should call it ``ODE"). Brownian motion is the key ingredient to enrich the representability of our directed-chain systems.
\end{rem}


\section{Experimental Details}\label{app:experiment}
Both discriminative and predictive metrics involve training tasks, and we shall first list all implementing details of these metrics, which is universal for all experiments. Then, we provide training hyper-parameters and training details used in different experiments.
\subsection{Metrics}\label{app:experiment_metrics}
\paragraph{Discriminative Metric.}
We first generate the same amount of fake data paths as true data paths to avoid imbalance, and choose 80\% from both real and fake data as training data, leaving the rest 20\% as testing data. We use a two-layer LSTM classifier with \texttt{channels/2} as the size of the hidden state, where \texttt{channels} is the dimension of generated and real series. We will minimize the cross-entropy loss, and the optimization is done by Adam optimizer with a learning rate of 0.001 for 5000 iterations. The discriminative score is calculated by the difference between 0.5 and the prediction accuracy on testing data.

\paragraph{Predictive Metric.}
We first generate the same amount of fake data as true data, and use it as training data for the predictive metric, whereas true data is for testing. We use a two-layer LSTM sequential predictor with \texttt{channels/2} as the size of the hidden state, where \texttt{channels} is the dimension of generated and real series. Our objective function is $L^1$ distance between predicted sequences and true sequences. The predictor generates one-step future predictions in the last feature with the others as input. Optimization is done by Adam optimizer with a learning rate of 0.001 for 5000 iterations. The predictive score is reported as the $L^1$ distance (also interpreted as mean absolute error (MAE)) between the predictive sequences and true sequences on testing data.

\paragraph{Independence Metric.}
The independence score is computed by the maximum of the $L^1$ distance of cross-correlation matrices over the time period $[0,T]$. In practice, we consider the maximum over the time stamps $t\in\{0.1,0.2,0.3,0.4,0.5,0.6,0.7,0.8,0.9,1.0\}$. 

\subsection{Experiments}\label{app:experiment_details}
In all four experiments, we use feed-forward neural networks with two hidden layers of sizes [128,128] to parameterize the drift $V_0$ and diffusion coefficient $V_1$. For the purpose of fair comparisons, we use the same GAN structure for both neural SDEs and DC-GANs, i.e., the same Sig-Wasserstein GAN setup \eqref{eq:sig_mmd} or the combination of neural CDE and Sig-WGAN scheme \eqref{eq:combined_loss} as discriminators. We remark that DC-GANs can be adapted to \texttt{torchsde}\footnote{See the Python package \url{https://github.com/google-research/torchsde}.} framework and use their adjoint method for back-propagation.

\paragraph{Stochastic Opinion Dynamics.}
In this experiment, we only use Sig-Wasserstein GAN approach for the discriminator, and choose $m = 8$ as the truncation depth in \eqref{eq:sig_mmd}. We choose $N = 1$ and  $d = 3$ dimensional standard Brownian motion in the DC-GANs generator \eqref{eq:dc_sde-simulation}, a batch size of 1024, a learning rate of 0.001 decaying to one-tenth for every 500 steps, and train a total of 2000 steps. Training data and testing data are sampled by the Euler scheme \eqref{eq:dc_sde-simulation} with a sample size of 8192, and their initial distributions $\xi$ are drawn independently from a uniform distribution on $[-2, 2]$. 

\paragraph{Stochastic FitzHugh-Nagumo Model.}
The stochastic FitzHugh-Nagumo model is widely used in neuroscience for describing the neurons' interacting spiking, in particular, to capture the multimodality of neurons' interspike interval distribution. For $N$ neurons and $P$ different neuron populations, we denote by $p(i)=\alpha, \alpha\in\{1,\dots,P\}$, the population of $i$-th particle belongs to,  for $i\in\{1,\dots,N\}$. The state vector $(X_t^{t, N})_{t\in[0,T]}=(V_t^{i,N}, w_t^{i,N}, y_t^{i,N})_{t\in[0,T]}$ of neural $i$ follows a three-dimensional SDE:
\begin{align*}
    \ud X_t^{t, N} &= f_\alpha(t, X_t^{t, N})\ud t + g_\alpha(t, X_t^{t, N})\left[ \begin{array}{c}
         \ud W_t^i  \\
         \ud W_t^{i,y} 
    \end{array} \right] \\
    & \quad+ \sum_{\gamma=1}^P\frac{1}{N_\gamma}\sum_{j,p(j)=\gamma}\bigg( b_{\alpha\gamma}(X_t^{i,N}, X_t^{j,N})\ud t + \beta_{\alpha\gamma}(X_t^{i,N}, X_t^{j,N})\ud W_t^{i,\gamma} \bigg),
\end{align*}
where $N_\gamma$ denotes the number of neurons in population $\gamma$. For all $\gamma$ and $\alpha \in \{1,\dots, P\}$, $I^\alpha(t):=I$, $\forall t\in[0,T]$, $\forall \alpha$ for some constant value $I$, $f_\alpha$, $g_\alpha$, $b_{\alpha\gamma}$ and $\beta_{\alpha\gamma}$ are given by
$$f_\alpha(t, X_t^{i,N}) = \left[\begin{array}{c}
     V_t^{i,N}-\frac{(V_t^{i,N})^3}{3}-w_t^{i,N} + I^\alpha(t)  \\
     c_\alpha(V_t^{i,N} + a_\alpha - b_\alpha w_t^{i,N})  \\
     a^\alpha_r S_\alpha(V_t^{i,N})(1-y_t^{i,N}) - a_d^\alpha y_t^{i,N}
\end{array}\right], \quad g_\alpha(t, X_t^{i,N}) = \left[ \begin{array}{cc}
    \sigma_{\mathrm{ext}}^\alpha & 0 \\
    0 & 0 \\
    0 & \sigma_\alpha^y(V_t^{i,N}, y_t^{i,N})
\end{array} \right],$$
and 
$$b_{\alpha\gamma}(X_t^{i,N}, X_t^{j,N})=\left[ \begin{array}{c}
    -\Bar{J}_{\alpha\gamma}(V_t^{i,N}-V^{\alpha\gamma}_{\mathrm{rev}})y_t^{i,N}  \\
    0 \\
    0
\end{array} \right],\quad \beta_{\alpha\gamma}(X_t^{i,N}, X_t^{j,N}) = \left[\begin{array}{c}
     -\sigma^J_{\alpha\gamma}(V_t^{i,N} - V_{\mathrm{rev}}^{\alpha\gamma})y_t^{i, N}  \\
     0 \\
     0 
\end{array}\right].$$
The functions $S_\alpha$, $\mc{X}$ and $\sigma_\alpha^y$ are defined as
\begin{align*}
    & S_\alpha(V_t^{i,N}) = \frac{T^\alpha_{\max}}{1+e^{-\gamma_\alpha(V_t^{i,N} - V_T^{i,N})}}, \\
    & \mc{X}(y_t^{i,N}) = \mathds{1}_{y_t^{i,N}\in(0,1)} \Gamma e^{-\Lambda/(1-(2y_t^{i,N}-1)^2)},\\
    &\sigma_\alpha^y(V_t^{i,N}, y_t^{i,N}) = \sqrt{a_r^\alpha S_\alpha(V_t^{i,N})(1-y_t^{i,N}) + a_d^\gamma y_t^{i,N}} \times  \mc{X}(y_t^{i,N}),
\end{align*}
where $(W^i, W^{i,y}, W^{i,\gamma}), i=1,\dots,N$ are standard three-dimensional Brownian motions that are mutually independent. For sample paths produced by this model, we follow the parameter choices in line with \citet{reis2018},
$$
\begin{array}{ccccccc}
    V_0=0, & \sigma_{V_0} =0.4, & a=0.7, & b=0.8,& c=0.08, & I=0.5, & \sigma_{ext}=0.5, \\
    w_0=0.5, & \sigma_{w_0}=0.4, & V_{rev}=1, & a_r=1, & a_d=1, & T_{max}=1, &\lambda=0.2, \\
    y_0=0.3, & \sigma_{y_0}=0.05, & J=1, & \sigma_j=0.2, & V_T=2, & \Gamma=0.1,& \Lambda=0.5.
\end{array}
$$
The above choice produces the joint multimodal distribution of $V$ and $w$; see the figure below.
\begin{figure}[htpb]
    \centering
    \vspace{0.1in}
    \includegraphics[width=0.35\textwidth, trim = {6em 1em 6em 4em}, clip, keepaspectratio=True]{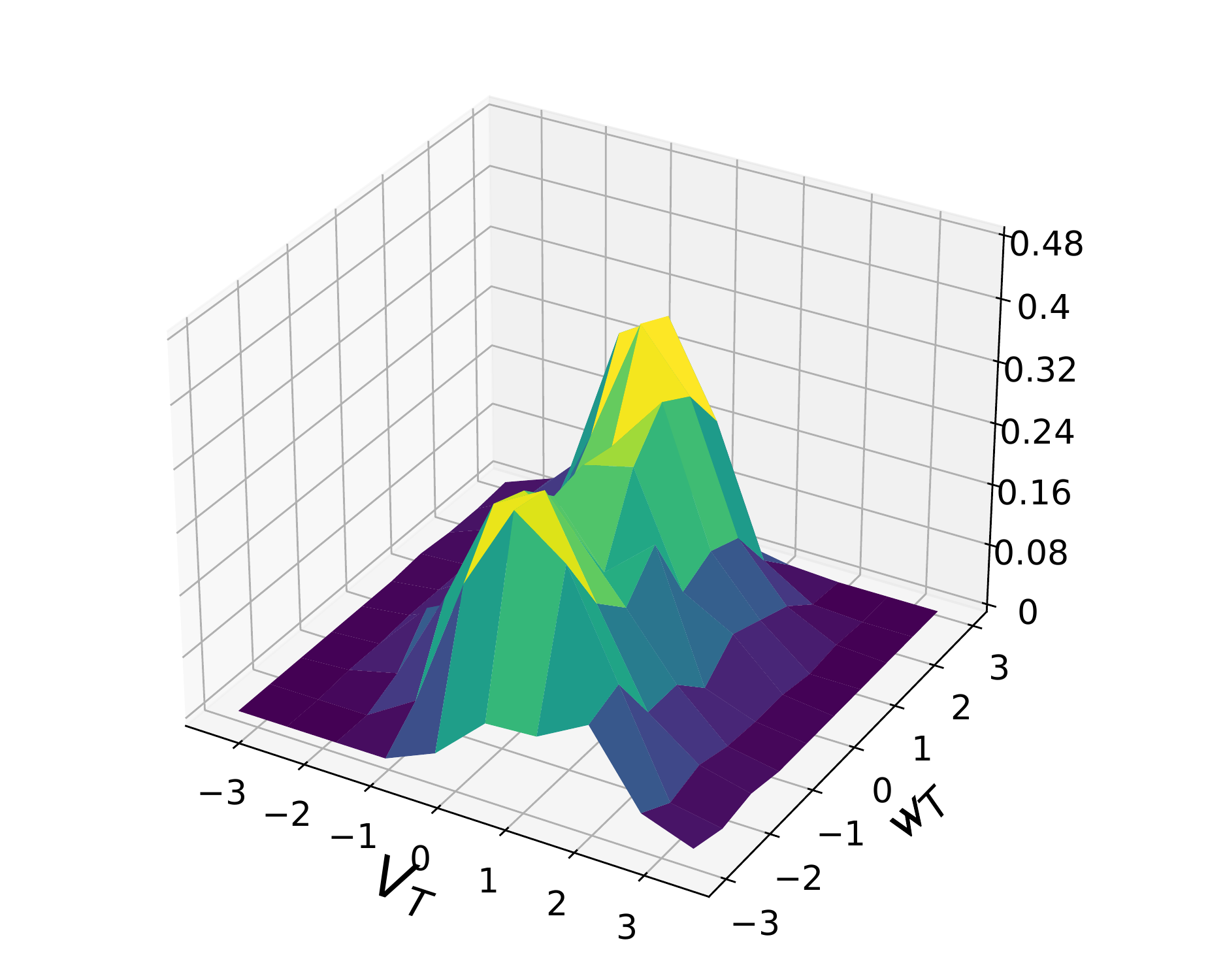}
    \hspace{30pt}
    \includegraphics[width=0.35\textwidth]{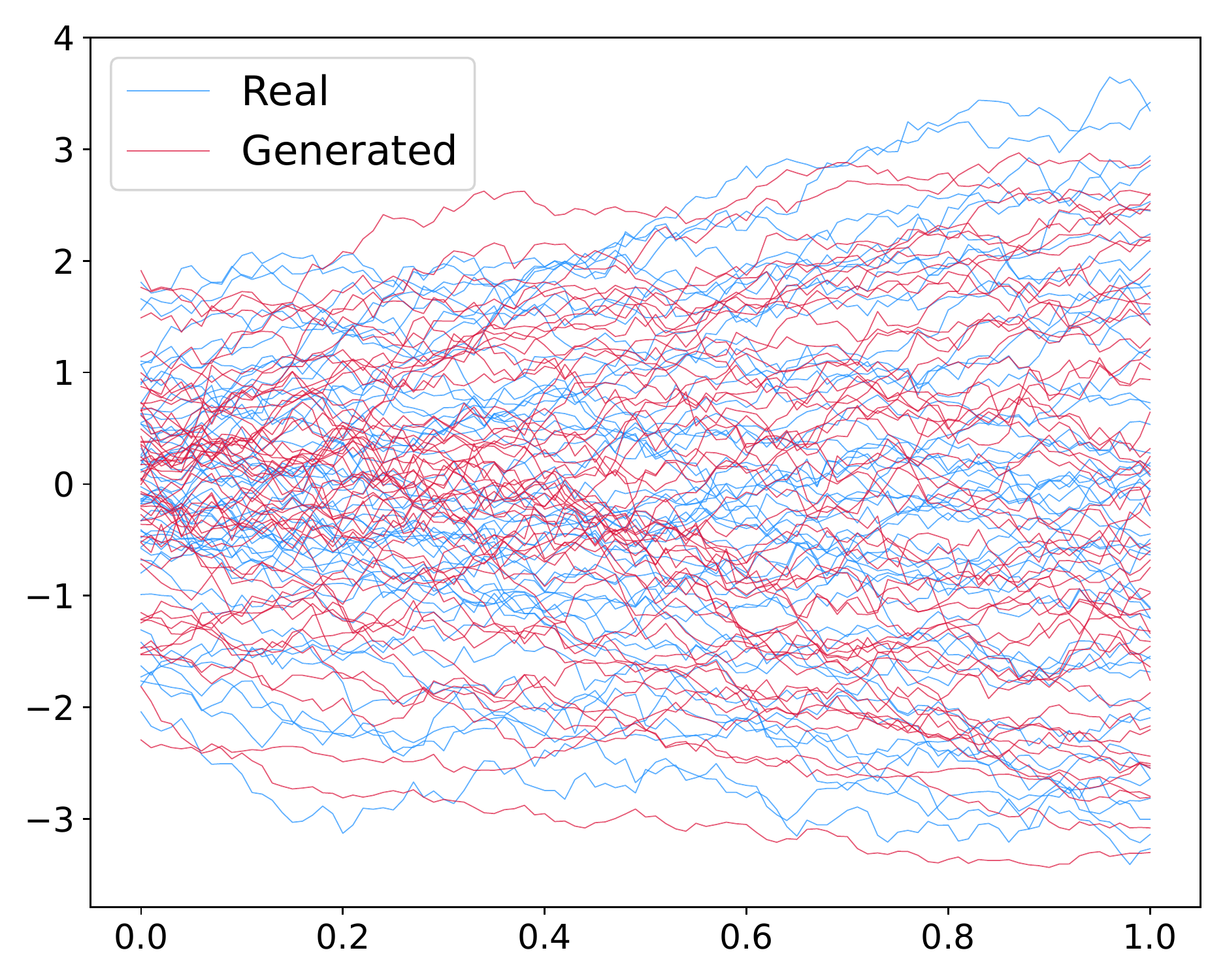}
    \vspace{0.2in}
    \caption{Stochastic FitzHugh-Nagumo Model (Example 2). Left subfigure shows the multimodal joint density of $V_T$ and $w_T$, and right subfigure shows the sample paths from the time-dependent model (blue) and from the DC-GANs (red).}
    \label{fig:fitzhugh-dist}
\end{figure}

All training and testing data are generated through the Euler scheme with the above parameters. In the training phase, we choose the Sig-Wasserstein GAN approach again for our discriminator and choose $m = 6$ as the truncation depth in \eqref{eq:sig_mmd}. We take $N = 3$ and $d = 5$ in our DC-GANs generator, and use a batch size of 1024 for training 2000 steps with a learning rate of 0.001 decaying to one-tenth every 500 steps. Training and testing data are generated by the Euler scheme, where the initial positions $\xi$ are drawn from a 3-dimensional Gaussian random variable with means $(0, 0.5, 0.3)$ and standard deviations $(0.4, 0.4, 0.05)$. 

\paragraph{Stock Price Time Series.}
In this real-world example, we use the six-dimensional stock price data of Google from 2004 to 2019. We segment them into sequences of length 24, which results in 3773 sequences as our time series data set. The combination of Neural CDEs and Signature MMD \eqref{eq:combined_loss} is used as the discriminator. For the purpose of a fair comparison, we use the same noise size $d$ and discriminator setup for both Neural SDEs and DC-GANs generators. In particular, for the Neural CDEs discriminator,  we set the dimension of the hidden process to be 16, and their coefficients are approximated by a feed-forward neural network with two hidden layers of size [128, 128]. For both DC-GANs and Neural SDEs generator, the Brownian motion's dimension is set at $ d = 10$; and for the Neural SDEs which embed stock prices data into a hidden space, we set its (the hidden space) dimension at 12. The batch size is chosen to be 128. Both generators and discriminators are trained using Adam optimizer. Both learning rates start at 0.0001 and decay to one-tenth after 2000 steps, the signature depth is chosen at $m =4$ in \eqref{eq:combined_loss} to alleviate the dimensional burden, and training steps are set as 4000. Our CTFP implementation follows the setup in \citet{deng2020modeling}, SigWGAN follows from \citet{ni2021sig} and TimeGAN implementation follows the setup in \citet{yoon2019time}.

{
\paragraph{Energy Consumption.}
In the real-world energy consumption example, we choose four electric and gas consumption time series from 02/2011-02/2013 and use daily data as a single time series, bringing 694 sequences with a length of 96. For both neural SDEs and DC-GANs, we use a ten-dimensional Brownian motion and neural nets with two hidden layers of size [128, 128] to estimate drift and diffusion coefficients. The batch size is 128, the training step is 4000, and the learning rate for the generator starts at 0.0001. In the case of using Neural CDEs as the discriminator, we use hidden size 16, [128, 128] as the hidden layers of neural nets estimating coefficients and 0.0001 as the starting learning rate of the discriminator. In the case of using SigWGAN, we consider signature depth 6. All learning rates decay to one-tenth after 2000 steps. Our CTFP implementation follows the setup in \citet{deng2020modeling}, SigWGAN follows from \citet{ni2021sig} and TimeGAN follows from \citet{yoon2019time}.
}


\end{document}